\documentclass{article}

\usepackage{microtype}
\usepackage{graphicx}
\usepackage{subfigure}
\usepackage{booktabs} % for professional tables
\usepackage{todonotes}
\usepackage{hyperref}

\usepackage[accepted]{icml2024}

\usepackage{amsmath}
\usepackage{amssymb}
\usepackage{mathtools}
\usepackage{amsthm}
\usepackage{amsfonts}
\usepackage{bbm}

\usepackage[capitalize,noabbrev]{cleveref}

\theoremstyle{plain}
\newtheorem{theorem}{Theorem}[section]
\newtheorem{proposition}[theorem]{Proposition}
\newtheorem{lemma}[theorem]{Lemma}
\newtheorem{corollary}[theorem]{Corollary}
\theoremstyle{definition}
\newtheorem{definition}[theorem]{Definition}

\theoremstyle{remark}
\newtheorem{remark}[theorem]{Remark}

\usepackage{todonotes}

\icmltitlerunning{Convex Relaxations of ReLU Neural Networks Approximate Global Optima in Polynomial Time}
\begin{document}

\twocolumn[
%\icmltitle{Submission and Formatting Instructions for \\
%           International Conference on Machine Learning (ICML 2024)}
\icmltitle{Convex Relaxations of ReLU Neural Networks Approximate Global Optima in Polynomial Time}



\icmlsetsymbol{equal}{*}

\begin{icmlauthorlist}
%\icmlauthor{Paul Chu}{equal,yyy}
\icmlauthor{Sungyoon Kim}{yyy}
\icmlauthor{Mert Pilanci}{yyy}
%\icmlauthor{Yifei Wang}{equal, yyy, comp}
%\icmlauthor{Firstname5 Lastname5}{yyy}
%\icmlauthor{Firstname6 Lastname6}{sch,yyy,comp}
%\icmlauthor{Firstname7 Lastname7}{comp}
%\icmlauthor{Firstname8 Lastname8}{sch}
%\icmlauthor{Firstname8 Lastname8}{yyy,comp}
\end{icmlauthorlist}

\icmlaffiliation{yyy}{Department of Electrical Engineering, Stanford University, California, United States}
%\icmlaffiliation{comp}{Company Name, Location, Country}
%\icmlaffiliation{sch}{School of ZZZ, Institute of WWW, Location, Country}

\icmlcorrespondingauthor{Sungyoon Kim}{sykim777@stanford.edu}
%\icmlcorrespondingauthor{Firstname2 Lastname2}{first2.last2@www.uk}

\icmlkeywords{Optimization Landscape, Training Guarantees, Random Cones, ICML}

\vskip 0.3in
]



\printAffiliationsAndNotice{}  % leave blank if no need to mention equal contribution

\begin{abstract}

In this paper, we study the optimality gap between two-layer ReLU networks regularized with weight decay and their convex relaxations. We show that when the training data is random, the relative optimality gap between the original problem and its relaxation can be bounded by a factor of $O(\sqrt{\log n})$, where $n$ is the number of training samples. A simple application leads to a tractable polynomial-time algorithm that is guaranteed to solve the original non-convex problem up to a logarithmic factor. Moreover, under mild assumptions, we show that local gradient methods converge to a point with low training loss with high probability. Our result is an exponential improvement compared to existing results and sheds new light on understanding why local gradient methods work well. 

%This document provides a basic paper template and submission guidelines.
%Abstracts must be a single paragraph, ideally between 4--6 sentences long.
%Gross violations will trigger corrections at the camera-ready phase.
\end{abstract}

\section{Introduction}
\label{1.intro}

After the tremendous success of deep learning \cite{lecun2015deep}, data-driven approaches have become a prominent trend in various areas of computer science. Perhaps a surprising fact is that despite the highly non-convex landscape of deep learning models \cite{li2018visualizing}, local gradient methods such as stochastic gradient descent (SGD) \cite{bottou2010large} or ADAM \cite{kingma2014adam} find nearly-global minimizers of the network extremely well. This mystery has gained wide attention in the learning theory community, and many have worked on the problem of proving convergence results of local methods under certain assumptions such as the infinite-width limit \cite{jacot2018neural}, heavy overparametrization 
 \cite{du2019gradient}, \cite{arora2019fine},\cite{zou2020gradient}, \cite{zou2018stochastic}, and milder width assumptions \cite{ji2019polylogarithmic}.

\begin{figure}[htbp]
\label{approx}
    \centering
    \includegraphics[width=0.44\textwidth]{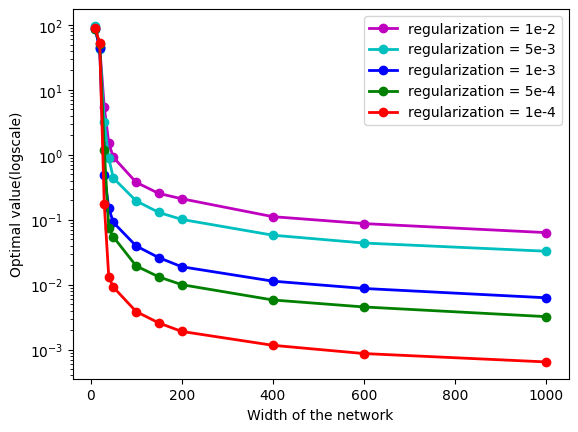}
    \caption{Convex relaxations of different widths. Here, we show the optimal value of the relaxed problem for different numbers of subsampled hyperplane arrangement patterns and regularization. Note that $n = 300, d = 10$, hence there are $\approx 30^{10}$ variables in the convex reformulation. However, using $\approx 30$ neurons in the relaxed problem optimizes the objective well.}
    \label{figure1}
\end{figure}
In contrast to the empirical success of local gradient methods, training a simple two-layer neural network with ReLU activation is proven to be NP-Hard \cite{boob2022complexity}. The sharp contrast between existing works on trainability guarantees and hardness results occurs as an increase in the width $m$, though it may seem to make the problem harder, actually makes it easier. Hence, it is a natural interest to understand the critical width $m^{*}$ that guarantees polynomial time algorithms where $m \geq m^{*}$, and the corresponding algorithm.

Recently, it was shown that two-layer and deeper ReLU networks have exact convex reformulations \cite{pilanci2020neural, wang2022parallel}. However, the number of variables in these convex problems can be exponentially large. For a critical width $m^{*} \leq n+1$, when $m \geq m^{*}$, we have an algorithm that exactly solves the non-convex problem with complexity polynomial in $n$ (the number of data) and exponential in $r$ (the rank of the dataset). Such an algorithm is given by considering the convex reformulation of the original problem which has approximately $(n/r)^r$ variables and solving the reformulated problem with standard interior-point solvers. Due to the exponential factor in $r$, exactly solving the reformulated problem is intractable, and \cite{pilanci2020neural} proposed to solve the randomly relaxed version of the exact reformulation with $m\ll (n/r)^r$ variables, where $m$ is only polynomially large (see \cref{PRE} for details). Though the relaxed problem has polynomial complexity with respect to all dimensions and works well in practice (see \cref{figure1}), the optimality gap between the relaxed and the original problem has remained unknown. In \cref{figure1}, width corresponds to the number of variables that the convex relaxation has.

In this paper, we study the randomized relaxation of the exact convex reformulation and find that:
\begin{itemize}
    \item Under certain assumptions on the input data, the relative optimality gap between the non-convex problem and the randomized relaxation is bounded by a logarithmic factor. This shows that the convex relaxations run in polynomial-time and have strong approximation properties. 
    \item Using results from \cite{wang2021hidden} and \cite{li2018learning}, we show that with high probability, local gradient methods with random initialization converges to a stationary point that has $O(\sqrt{\log n})$ relative training error with respect to the global minimum.
    \item We propose a tractable polynomial-time algorithm which is polynomial in all dimensions that can approximate the global optimum within a logarithmic factor. We are not aware of any similar result for regularized ReLU networks.
\end{itemize}
The paper is organized as follows: in \cref{PR}, we go through a brief overview of related works. \cref{MR} is devoted to a detailed description of the contributions we made. \cref{UCR} and \cref{CR} illustrate the overall proof strategy, novel approximation guarantees, and the geometry behind it: we first obtain approximation guarantees of an ``easier version" of the randomly relaxed problem, and then solve the relaxed problem using the result. We wrap the paper with \cref{Conclu}, presenting possible further discussions.

\subsection{Prior and Related Work}
\label{PR}
\textbf{Convex reformulation of neural networks:}
Starting from \cite{pilanci2020neural}, an extensive line of works have discussed the convex reformulation of neural networks. The main idea that lies in this line of work is that for different neural network architecture, e.g. for CNNs \cite{ergen2020implicit}, transformers \cite{sahiner2022unraveling}, multi-layer networks \cite{ergen2021global}, vector outputs \cite{sahiner2020vector}, etc, we have an exact convex reformulation with a different regularizer. Hence, neural network architectures can be understood as imposing different regularizations on the convex problem.

Moreover, these convex reformulations have given interesting insights into the training of neural networks. For example, \cite{wang2021hidden} characterizes all stationary points of a two-layer neural network via global optimum of corresponding convex problems, and \cite{wang2023polynomial} discusses the intrinsic complexity of training two-layer networks with the view of convex duality and equivalence with the MAX-CUT problem.  

%\todo[color=red!40]{We can avoid writing the above equation twice (eq 2 and eq 3 ) by defining an index set like $\sum_{j\in \mathcal{J}}D_j X(u_j-v_j)$ and defining $\mathcal{J}$ differently}

%\todo[inline,color=red!40]{I think it's better to not mention Gated ReLU until the later sections that describe the proof strategy. Ideally, we should refrain introducing additional complexity to state our main result on ReLU networks. (We don't need Gated ReLU to state the main result, we only need it in the proof). The paper should flow really well, e.g., not too many new concepts, not too many equations, until we get to the main result. The section that contains the proof (or the outline of the proof) can have all the details like Gate ReLU etc.}

%%%%%%%%%%%%%%%%%%%%%%%%%% Text about nonconvex gated ReLU
%(\ref{nonconvex_gated}) can be understood as a variant of (\ref{nonconvex_original}) where applying the nonlinearity is disentangled from the optimization variable. $g_i$ s are prespecified vectors of size $\mathbb{R}^{d}$, referred to as ``gate vectors". It is clear that (\ref{gatedconvex}) is further equivalent to the problem 
%\begin{equation}
%\label{gatedconvex_simpler}
%%\min_{u_i \in \mathbb{R}^{d}} \frac{1}{2}\lVert \sum_{i=1}^{P} D_iXu_i - y \rVert_2^2 + \beta \sum_{i=1}^{P}\lVert u_i \rVert_2,
%\end{equation}
%due to triangle inequality. \cite{mishkin2022fast} uses (\ref{gatedconvex_simpler}) to approximate the original problem (\ref{nonconvex_original}), and use fast convex solvers to solve (\ref{gatedconvex_simpler}) efficiently.

\textbf{Training complexity of two-layer neural networks:} We use $n$ to denote the number of data, $d$ the dimension of data, and $m$ to be the width of the model. Several papers have discussed the theoretical complexity of training simple ReLU networks and exact methods to train them. It is proven that for single-width neural networks, it is NP-Hard to train a two-layer ReLU network \cite{boob2022complexity}, and even approximating the optimal error with $(nd)^{1/poly(\log \log(nd))}$ is proven to be NP-Hard \cite{goel2020tight}. \cite{dey2020approximation} gives an approximate algorithm to solve the single-width neural network problem with $O(n^k)$ complexity and $n/k$ approximate error for general input, and constant error bound in the student-teacher setting.

Some exact algorithms to train two-layer networks have been discussed, though due to the inherent complexity of training two-layer networks, these algorithms are often intractable in practice. \cite{arora2016understanding} presents an algorithm with $O(2^{m}n^{dm}poly(n, d, m))$ complexity, where $m$ is the width of the network. \cite{manurangsi2018computational} shows that when both the input and weights are constrained on a unit ball, we can train a two-layer network to have $\epsilon$ - error bound with $O((2^{m/\epsilon})^{O(1)}n^{O(1)})$ complexity. Training guarantees in the student-teacher setting with Gaussian input have also been discussed: \cite{bakshi2019learning} discusses Gaussian inputs and in the student-teacher setting, we can exactly obtain the teacher parameters in $poly(n)$ time, and \cite{awasthi2021efficient} discusses the same problem and propose an algorithm that works in $poly(n,m,d)$. 

\textbf{Note that the aforementioned complexity results are on training problems without regularization}. When we have an $l_2$ regularized problem, the best exact solution to the problem is presented in \cite{pilanci2020neural}, where they propose an algorithm of $O(d^3(n/d)^{3d})$ complexity provided that the width of the network is sufficiently large. Extending the idea, \cite{bai2023efficient} proposes an approximate method with complexity $O(d^2m^2)$ that works for $m \geq n/\xi$ for a predetermined error threshold $\xi$. However, their results require a very wide neural network and the relative optimal error bound is unclear. Note that choosing an appropriate regularization affects the performance of the model during test time and considering the regularized setting is not only for theoretical interest (see \cref{appendix:regularization} for experiments). 

\textbf{Guarantees of local gradient methods for two-layer neural networks:} An extensive line of work has discussed why local gradient methods such as gradient descent and its variants work so well for neural networks - even in the simplest setting of two layers. \cite{soudry2016no}, \cite{soudry2017exponentially}, \cite{tian2017analytical} analyzes the local minima of neural networks with a similar flavor to our work: \cite{soudry2016no} shows when the number of parameters for a single layer exceeds $n$, differentiable local minima are global, and \cite{soudry2017exponentially} proves that differentiable regions that contain sub-optimal local minima are exponentially small compared to the ones containing global minima. \cite{tian2017analytical} exploits the closed-form formula of population gradients to characterize the region of suboptimal critical points. Our work extends these works, for networks with regularization and non-differentiable stationary points under certain assumptions.

When Gaussian input is assumed, many works have analyzed whether local gradient methods can recover true parameters in the student-teacher setting. To prove that local gradient methods can recover true parameters, existing works either use a specifically tailored multi-phase analysis of gradient methods with deliberately chosen parameters and initialization \cite{li2017convergence}, \cite{du2018gradient}, \cite{zhou2019toward}, \cite{bao2024global}, or choose a particular structure of the model, e.g. no overlapping CNNs \cite{brutzkus2017globally}, \cite{zhang2020improved}, two-layer network plus a skip connection \cite{li2017convergence}, fixed second layer weights \cite{du2018gradient}, \cite{zhou2021local}. Our result is more abstract in the sense that it works for any local gradient method and initialization schemes that satisfy assumption (A2) and works in settings that are not student-teacher settings, though the analysis only gives relative approximation guarantees of these methods with respect to the global optimum.

Training guarantees have also been established for two-layer neural networks with different loss functions. Regarding hinge loss, \cite{brutzkus2017sgd}, \cite{laurent2018multilinear}, \cite{wang2019learning}, \cite{wang2023polynomial} have analyzed the loss landscape of hinge loss under the assumption that the data is linearly separable. In particular, \cite{laurent2018multilinear} discusses that all local minima are global for networks with leaky-ReLU activation, when we have hinge loss and the data is linearly separable, and \cite{wang2019learning} proposes a simple SGD-like algorithm that provably converges to a global minimum. \cite{wang2023polynomial} is worth noting that they have a similar convex analysis to ours to analyze the training hardness and approximation guarantees of training with hinge loss. Their guarantees are a different version of our result for hinge loss. 

For useful lemmas that will be used throughout the paper, see \cref{usefullemma}.
\section{Main Results}
\label{MR}
\subsection{Preliminaries}
\label{PRE}
First, we go through a formal description of the optimization problems that we are interested in. Let the data matrix $X \in \mathbb{R}^{n \times d}$, the label vector $y \in \mathbb{R}^{n}$, and weight decay regularization $\beta > 0$. Consider the training problem
\begin{equation}
\label{nonconvex_original}
p^{*}:=\min_{u_j, \alpha_j} \frac{1}{2}\lVert \sum_{j=1}^{m} (Xu_j)_{+} \alpha_j - y \rVert_2^2 + \frac{\beta}{2} \sum_{j=1}^{m} (\lVert u_j \rVert_2 ^2 + \lVert \alpha_j \rVert_2^2).
\end{equation}
The convex counterpart of problem (\ref{nonconvex_original}) can be written as in \cite{pilanci2020neural}:
\begin{multline}
\label{convex1}
p_0^{*}:=\min_{u_i,v_i \in \mathcal{K}_{D_i}} \frac{1}{2}\lVert \sum_{D_i \in \mathcal{D}} D_iX(u_i - v_i) - y \rVert_2^2 \\+ \beta \sum_{i \in \mathcal{I}} (\lVert u_i \rVert_2 + \lVert v_i \rVert_2),
\end{multline}
where $\mathcal{D}$ is a set of hyperplane arrangement patterns $diag(\mathbbm{1}[Xu \geq 0])$ and $\mathcal{K}_{D} = \{u \ |\ (2D - I)Xu \geq 0\}$ are cones where each optimization variable are constrained at. The intuition behind the convex counterpart is that as the nonconvexity of problem (\ref{nonconvex_original}) comes from the nonlinearity and multiplying two variables $u_j$ and $\alpha_j$, we linearize the model and merge the two variables with appropriate scaling to make the problem convex \cite{mercklé2024thehiddenconvex}.

\textbf{Convex reformulation:}
Let's say $[n] = \{1, 2, ..., n\}$, $D_i$ from $i \in [P]$ denote all possible hyperplane arrangement patterns, $\{u_i^{*}, v_i^{*}\}_{i=1}^{P}$ are global minima of problem (\ref{convex1}), 
$$
m^{*} = \sum_{i=1}^{P} \mathbbm{1}[u_i^{*}\neq 0] + \mathbbm{1}[v_i^{*} \neq 0],
$$ and $m \geq m^{*}$. Problems (\ref{nonconvex_original}) and (\ref{convex1}) then become equivalent, i.e. $p^{*} = p_0^{*}$ and we can construct an optimal solution of problem (\ref{nonconvex_original}) with $\{u_i^{*}, v_i^{*}\}_{i=1}^{P}$. This means when $\mathcal{D} = \{D_i\ |\ i \in [P]\}$, we can exactly solve the original problem $(\ref{nonconvex_original})$ with its convex reformulation $(\ref{convex1})$. Though $m^{*}$ might seem exponential in $n$ as $P$ is exponentially large in $n$, an application of Caratheodory's theorem shows that $m^{*} \leq n + 1$.

\textbf{Gaussian relaxation:}
From \cite{4038449}, we know that $P = O((\frac{n}{r})^r)$ where $r = rank(X)$. Therefore, it is computationally intractable to sample all possible hyperplane arrangement patterns for large $r$. The randomized Gaussian relaxation of the problem (\ref{convex1}) is where
$$
\mathcal{D} = \{\Tilde{D}_i = diag(\mathbbm{1}[Xg_i \geq 0])\ |\ g_i \sim \mathcal{N}(0, I_d), i \in [\Tilde{P}]\},
$$ 
i.e. instead of using every possible hyperplane arrangement pattern, we randomly sample $\Tilde{P}$ patterns given by a random Gaussian vector. Note that we are using a very small portion of the set of all hyperplane arrangement patterns, and it is not trivial that we will get good approximation guarantees.

Next, we illustrate an assumption on the data distribution for the main results.

\textbf{Assumption on training data:} Throughout the paper, we assume that the data distribution follows $X_{ij} \sim \mathcal{N}(0,1)$ i.i.d., $n/d = c \geq 1$, where $c$ is a constant, and denote this assumption as (A1). (A1) may be extended to a distribution with rotational symmetry, controllable quantile, and where Gordon's comparison \cite{thrampoulidis2014gaussian} or empirical process results \cite{mendelson2007reconstruction} are applicable. A related analysis can be found in \cite{thrampoulidis2015isotropically}, where they extend Gordon's comparison to isotropic random orthogonal matrices. Also, a connection to restricted isometry property \cite{candes2008restricted} could be a key to extending the result to different distributions. Here, we assume Gaussianity for simplicity.\\ Also, the regime $n/d = c$, e.g. $n \asymp d$, is extensively studied in the literature \cite{montanari2019generalization}, \cite{celentano2021high}, \cite{celentano2022fundamental}. The given structure enables a quantitive comparison between the randomized relaxation and the original problem with random matrix theory and provides better bounds than arbitrary inputs.

\subsection{Overview of Theoretical Results}

We are now ready to state the main results of the paper. The first result is the optimality bound between the non-convex problem \eqref{nonconvex_original} and its Gaussian relaxation, assuming (A1). 

\begin{theorem}
\label{finalthm}
(Informal) Consider the two-layer ReLU network training problem
$$
p^{*}:=\min_{u_j, \alpha_j} \frac{1}{2}\lVert \sum_{j=1}^{m} (Xu_j)_{+} \alpha_j - y \rVert_2^2 + \frac{\beta}{2} \sum_{j=1}^{m} (\lVert u_j \rVert_2 ^2 + \lVert \alpha_j \rVert_2^2),
$$
and its convex relaxation
\begin{multline}
\label{approxconvex}
\Tilde{p}^{*}:=\min_{u_i,v_i \in \mathcal{K}_{\Tilde{D}_i}} \frac{1}{2}\lVert \sum_{\Tilde{D}_i \in \mathcal{D}} D_iX(u_i - v_i) - y \rVert_2^2 \\+ \beta \sum_{i \in \mathcal{I}} (\lVert u_i \rVert_2 + \lVert v_i \rVert_2).
\end{multline}
Here, $|\mathcal{D}| = m/2$ and the elements are sampled by the hyperplane arrangement patterns of random Gaussian vectors. Assume (A1), $d$ is sufficiently large, and suppose $m = \kappa \max\{m^{*}, 320(\sqrt{c}+1)^2 \log(\frac{n}{\delta})\}$ for some fixed $\kappa \geq 1$. Then, with high probability,
$$
p^{*} \leq \Tilde{p}^{*} \leq C\sqrt{\log 2n}\ p^{*}.
$$
for some constant $C \geq 1$.
\end{theorem}
\begin{remark}
    To the best of our knowledge, the above result provides the first polynomial-time approximation guarantee for regularized ReLU NNs. Also note that typically we have $p^*\rightarrow 0$ as $\beta\rightarrow 0$, implying $\tilde p^*-p^*\rightarrow 0$. 
    %Also note that $p^*,\tilde p^*\rightarrow 0$ when $\beta\rightarrow 0$.
\end{remark}
For a detailed statement of the theorem, see \cref{Proofmain}. A few points on \cref{finalthm} is worth mentioning. First, it gives a guarantee of the convex relaxation of the original problem when we sample only $\max\{m^{*}, O(\log n)\}$ hyperplane arrangement patterns. This is an exponential improvement over existing convex reformulations, and as $m^{*}$ is much smaller than $n+1$ in practice, the width bound for \cref{finalthm} is practical. Also, with such characterization, we obtain a polynomial-time approximate algorithm that works with guarantees by simply solving the convex relaxation with standard interior-point solvers \cite{potra2000interior}.

\begin{theorem}
Assume (A1), $d$ is sufficiently large, and suppose that $m = \kappa \max\{m^{*}, 320(\sqrt{c}+1)^2\log(n/\delta)\}$ for some fixed $\kappa \geq 1$. Then, there exists a randomized algorithm with $O(d^3m^3)$ complexity that solves problem $(\ref{nonconvex_original})$ within $O(\sqrt{\log n})$ relative optimality bound with high probability.
\end{theorem}

Moreover, with the characterization of stationary points in \cite{wang2021hidden}, we can prove that local gradient methods converge to ``nice" stationary points under the assumption below:

(A2) While local gradient method iterates, only a randomly chosen portion $p < 1/2$ of the initial hyperplane arrangement patterns change until it converges.
\begin{figure}
    \centering
     \begin{subfigure}
         \centering
         \includegraphics[width=0.35\textwidth]{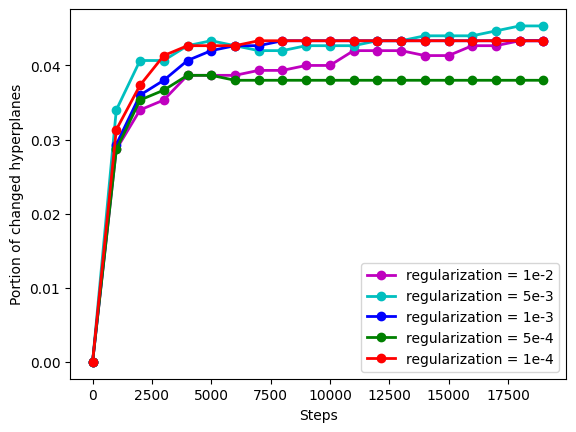}
     \end{subfigure}
     \hfill
     \begin{subfigure}
         \centering
         \hspace*{0.2cm}
         \includegraphics[width=0.35\textwidth]{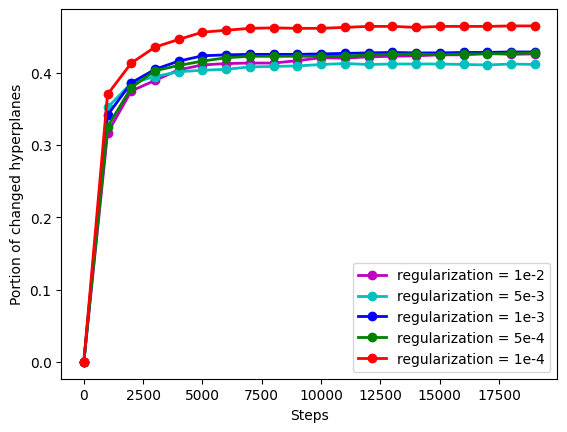}
     \end{subfigure}
        \caption{Verification of (A2) for gradient descent. The upper figure shows how many hyperplane arrangement patterns change for random data, and the lower figure shows how many of them change for MNIST.}
        \label{a2}
\end{figure}
(A2) is also considered in \cite{li2018learning}, where they show that SGD with random initialization preserves most of the hyperplane arrangement patterns. The specific lemma, which they refer to as the coupling lemma, is proven for cross-entropy loss. It is also verified for gradient descent with MSE loss, both for random data and MNIST in \cref{a2}.
\begin{lemma}
(Coupling lemma, \cite{li2018learning}) With high probability over the random initialization that follows $\mathcal{N}(0, m^{-1} I)$, and suppose the data is generated from $l$ underlying distributions. For every $\tau > 0$, $t = \Tilde{O}(\tau/\eta)$, at least $1 - e\tau l/\sigma$ portion of the hyperplane arrangement patterns remain the same.
\end{lemma}

The idea is that we can map each stationary point to a global minimum of a convex problem, and as we initialize at random, we can think of the stationary points as global minimizers of the convex problem with randomized hyperplane arrangement patterns. 
\begin{theorem}(Informal)
\label{losslandscape}
Consider the training problem $\min_{u_j, \alpha_j} \mathcal{L}(u, \alpha),$ where the loss function $\mathcal{L}$ is given as 
$$
\mathcal{L}(u, \alpha) = \frac{1}{2} \lVert \sum_{j=1}^{m} (Xu_j)_{+}\alpha_j - y \rVert_2^2 + \frac{\beta}{2} \sum_{j=1}^{m} (\lVert u_j \rVert_2^2 + \alpha_j^2).
$$
Assume (A1),(A2) and $d, m$ are sufficiently large. For any random initialization $\{u_i^{0},\alpha_i^{0}\}_{i=1}^{m}$, suppose local gradient method converged to a stationary point $\{u_i',\alpha_i'\}_{i=1}^{m}$. Then, with high probability,
$$
\mathcal{L}(u',\alpha') \leq C\sqrt{\log 2n}\ \mathcal{L}(u^{*},\alpha^{*}),
$$
for some $C \geq 1$. Here, $\{u_i^{*},\alpha_i^{*}\}_{i=1}^{m}$ is a global optimum of $\mathcal{L}(u, \alpha)$.
\end{theorem}
\begin{remark}
The above result shows that there exists many stationary points whose objective value is a logarithmic approximation of the global optimum. Therefore, first-order optimizers such as SGD and Adam can approximate the global optimum even when they converge to stationary points.
\end{remark}
Note that our analysis is not based on explicit iteration of local gradient methods, or does not exploit a NTK - based analysis. Rather, it follows from a simple yet clear analysis based on the lens of convex optimization.

\subsection{Notations for Proof}
\label{notations}
Before discussing the proof strategy, we clarify some frequently used variables in this section. We use $D_i$ for $i \in [P]$ to denote all possible hyperplane arrangements, $P$ to denote the number of all possible hyperplane arrangements, $\Tilde{D}_i$ to denote randomly selected hyperplane arrangement patterns, and $\Tilde{P}$ the number of such subsamples. We also use $\mathcal{K}_{D}$ as in preliminaries. $\mathcal{M}$ is used to denote
$$
\mathbb{E}_{g \sim \mathcal{N}(0,I_d)}[ diag[\mathbbm{1}(Xg \geq 0)] XX^{T} diag[\mathbbm{1}(Xg \geq 0)]],
$$
and $\kappa$, which is analogous to the condition number, to denote
$$
\frac{\lambda_{max}(XX^{T})}{\lambda_{min}(\mathcal{M})},
$$
provided that $\mathcal{M}$ is invertible. At last, $m^{*}$ is used to denote the number of nonzero variables for the optimal solution of the convex reformulation.
\subsection{Overall Proof Strategy}
To prove \cref{finalthm}, we first consider the unconstrained problem
\begin{multline}
\label{gatedconvex}
\Tilde{p_1}^{*}:=\min_{u_i,v_i \in \mathbb{R}^{d}} \frac{1}{2}\lVert \sum_{\Tilde{D}_i \in \mathcal{D}} \Tilde{D}_iX(u_i - v_i) - y \rVert_2^2 \\+ \beta \sum_{i \in \mathcal{I}} (\lVert u_i \rVert_2 + \lVert v_i \rVert_2).
\end{multline}
The convex problem \eqref{gatedconvex} was first introduced in \cite{mishkin2022fast}, where they first solve the unconstrained problem and decompose it into cone constraints to solve the original problem \eqref{approxconvex}. The problem is further equivalent to 
\begin{equation}
\label{gatedconvex_simpler}
\min_{u_i \in \mathbb{R}^{d}} \frac{1}{2}\lVert \sum_{\Tilde{D}_i \in \mathcal{D}} \Tilde{D}_iXw_i - y \rVert_2^2 + \beta \sum_{i \in \mathcal{I}}\lVert w_i \rVert_2,
\end{equation}
due to triangle inequality. We use a similar strategy that was introduced in \cite{mishkin2022fast}: first, we show that the unconstrained problem enjoys good approximation guarantees, even though we use only a subsample of the whole possible hyperplane arrangement patterns. After that, for global optimum $\{w_i^{*}\}_{i=1}^{m}$ of problem \eqref{gatedconvex_simpler}, we find $u_i, v_i \in \mathcal{K}_{D_i}$ that satisfies $u_i - v_i = w_i^{*}$ and minimal norm sum $\lVert u_i \rVert_2 + \lVert v_i \rVert_2$. We find that this decomposition does not increase the norm up to $O(\sqrt{\log 2n})$ factor, finally leading to the guarantee of the original problem. A novel duality-based analysis of the unconstrained problem and obtaining a relative optimality gap (\cref{c3.10}) and an analysis using Gordon's comparison to understand how ``sharp" each cone may be (\cref{conesharp}) are the major technical contributions of our paper.

\section{Guarantees for the Unconstrained Convex Relaxation}
\label{UCR}
Let's start with the unconstrained optimization problem \cref{gatedconvex_simpler}. Note $\Tilde{p}_1^{*}$ to be the optimal value of the random relaxation of the unconstrained problem, and $p_1^{*}$ to be the optimal value of the unconstrained problem using all possible hyperplane arrangement patterns. 
We wish to find constants $C_1, C_2>0$ depending only with the data matrix $X$ that satisfies
\begin{equation*}
0 \leq \Tilde{p}_1^{*} - p_1^{*} \leq C_1, \quad p_1^{*} \leq \Tilde{p}_1^{*} \leq C_2 p_1^{*}.
\end{equation*}
As a warmup, we first give results when the problem is unregularized. We can show that $p_1^{*} = \Tilde{p}_1^{*} = 0$ with high probability when we sample sufficiently many hyperplane arrangement patterns. After that, we give approximation results for a gated ReLU problem with $l_2$ regularization and use them to prove that  
$$
C_1 = \sqrt{2}\beta \frac{\lVert y \rVert_2}{\sqrt{\lambda_{min}(\mathcal{M})}}, C_2 = 2\sqrt{2} \sqrt{\frac{\lambda_{\max}(XX^{T})}{\lambda_{\min}(\mathcal{M})}},
$$
holds with probability at least $1 - \delta$, provided that we sample sufficiently many hyperplanes and the dimension $d$ is sufficiently large. We defer the proofs to \cref{proof3}.
\subsection{Warmup: Unconstrained Relaxation Without Regularization}
\label{3.1.Noreg}
To show that $p_1^{*} = 0$, we prove that we can approximate any vector with $2n$ hyperplane arrangement patterns. The proof is consistent with general overparametrization arguments, which state that when the width of the network is $m \geq n$, any local minimum becomes the global minimum. Here, we carefully choose $2n$ hyperplane arrangement patterns where each arrangement pair $D_{1,i}, D_{2,i}$ differ only at the $i$-th diagonal entry. With $n$ such pairs we can express any vector $y \in \mathbb{R}^{n}$ as a linear combination of vectors in the column space of $\{(D_{1,i}-D_{2,i})u|u \in \mathbb{R}^{n}\} = \{ke_i|k \in \mathbb{R}\}$, proving that there exists $u_i$ s satisfying
$$
\sum_{i=1}^{n} (D_{1,i} - D_{2,i})Xu_i = y.
$$
This means that if we have all possible hyperplane arrangement patterns we can express any vector $y$ with zero error. The specific proposition is as follows.
\begin{proposition}
\label{prop1}
Suppose no two rows of $X$ are parallel. There exists hyperplane patterns $\Tilde{D}_1$, $\Tilde{D}_2$, ..., $\Tilde{D}_{2n}$ such that for any $y \in \mathbb{R}^{n}$, there exists $u_1, u_2, ... u_{2n} \in \mathbb{R}^{d}$ satisfying
\begin{equation*}
\sum_{i=1}^{2n} \Tilde{D}_iXu_i = y.
\end{equation*}
\end{proposition}
However, we can further prove that using randomly sampled hyperplane arrangement patterns $\Tilde{D}_i$, we can fit arbitrary vector $y$ with probability at least $1 - \delta$, provided that we sample sufficiently many hyperplane arrangement patterns. 
\begin{proposition}
\label{prop2}
Suppose we sampled $\Tilde{P} = 2\kappa \log(\frac{n}{\delta})$
hyperplane arrangement patterns, provided that $\mathcal{M}$ is invertible. Then, with probability at least $1-\delta$, for any $y \in \mathbb{R}^{n}$ there exists $u_1, u_2, ..., u_{\Tilde{P}}$ satisfying
$$
\sum_{i=1}^{\Tilde{P}} \Tilde{D}_iXu_i = y.
$$
\end{proposition}
\cref{prop2} shows that when we sample sufficiently many hyperplane arrangement patterns, we can fit arbitrary vectors with high probability, hence $\Tilde{p}_1^{*} = 0$. The two propositions lead to the following approximation result.
\begin{corollary}
When $\beta = 0$ and $\Tilde{P} \geq 2\kappa \log(\frac{n}{\delta}),$
$p_0^{*} = p_1^{*} = 0$ with probability at least $1 - \delta$.
\end{corollary}

One natural question is ``how large will $\kappa$ be?" Interestingly, when (A1) is satisfied, we can show that $\kappa = O(n/d)$ holds (see \cref{3.5. kappa} for details). Hence, for random data, we use approximately $O(n\log n)$ parameters to fit to $y \in \mathbb{R}^{n}$ with high probability. Note that we need at least $n/d$ hyperplane arrangement patterns and $n$ parameters to fit to any vector. This means that for random data, \cref{prop2} is optimal up to a logarithmic factor.
\subsection{Unconstrained Relaxation with $l_2$ Regularization}
\label{3.2. l2reg}

In the case of $l_2$ regularization, we can find that when we sample more hyperplane arrangement patterns, the optimal value decreases with the same order as the number of planes with high probability. The proof is rather straightforward: we solve the linear-constrained quadratic problem, and then use matrix chernoff bounds to relate the eigenvalues of $\mathcal{M}$ with the eigenvalue of sample mean of matrices
$$
\mathcal{M}_{\Tilde{P}} = \frac{1}{\Tilde{P}} \sum_{i=1}^{\Tilde{P}} \Tilde{D}_i XX^{T} \Tilde{D}_i.
$$
\begin{theorem}
\label{approximation_l2reg}
Let $p_2^{*}$ be the optimal value of the unconstrained relaxed problem with $l_2$ regularization, and suppose $\Tilde{P} \geq 12\kappa \log(\frac{2n}{\delta})$. With probability at least $1-\delta$ there exists scalars $C_1, C_2 > 0$ that satisfies
$$
\frac{C_1}{\Tilde{P}} \leq p_2^{*} \leq \frac{C_2}{\Tilde{P}}.
$$
\end{theorem}
An intuitive explanation of \cref{approximation_l2reg} is that when we sample sufficiently many hyperplane arrangement patterns, we have $\mathcal{M}_{\Tilde{P}} \approx \mathcal{M}$ and the optimal value of the $l_2$ regularized problem concentrates to $\frac{\beta}{\Tilde{P}} y^{T}\mathcal{M}^{-1}y$. This characterization means that for $l_2$ regularization, we cannot have a relative optimality gap that is independent of $P$, the total number of all possible hyperplane arrangement patterns.

%A direct corollary of \cref{approximation_l2reg} shows that for C-GReLU with $l_2$ regularization, there cannot exist a constant $C$ irrelevent to the number of all possible hyperplane arrangement patterns $P$ that satisfies $\Tilde{p}_2^{*} \leq Cp_2^{*}$, as $C$ should be the order of $O(\frac{P}{\Tilde{P}})$. Therefore, in this case, we don't have a reasonable relative approximation error between the full C-GReLU and the approximated C-GReLU.
%\begin{corollary}
%Suppose $\Tilde{P}$ is fixed, $C \leq P/\Tilde{P}$. Then, with probability at least $1 - \exp(-c\Tilde{P})$, $\Tilde{p}_2^{*} \geq Cp_2^{*}$ holds.
%\end{corollary}
%Nevertheless, we can see that the absolute error between $p_2^{*}$ and $\Tilde{p}_2^{*}$ will decrease with the order of $O(\frac{1}{\Tilde{P}})$, meaning that the optimal values will actually be similar.

\subsection{Unconstrained Relaxation With Group $l_1$ Regularization}
\label{3.3. group l1reg}
Different from the case with $l_2$ regularization, when we have group $l_1$ regularization, we can have relative error bounds that are irrelevant to the total number of hyperplane arrangement patterns. We first obtain an upper bound of $\Tilde{p}_1^{*}$ by using surrogate variables $\kappa_i$ and solving the $l_2$ regularized problem,
\begin{equation}
\label{surrogatel2}
\min_{\kappa_i \in \mathbb{R}} \min_{u_i \in \mathbb{R}^{d}} \frac{1}{2}\lVert \sum_{i=1}^{\Tilde{P}} \Tilde{D}_iXu_i - y \rVert_2 ^2 + \frac{\beta}{2} \sum_{i=1}^{\Tilde{P}} (\kappa_i \lVert u_i \rVert_2^2 + \frac{1}{\kappa_i}).
\end{equation}
We cannot exactly solve problem (\ref{surrogatel2}). However, we can use the ansatz $\kappa_i = \frac{\kappa}{\Tilde{P}}$ and apply matrix concentration to upper bound the optimal value. 
\begin{proposition}
\label{prop3}
Let $\Tilde{p}_1^{*}$ the optimal value of the unconstrained relaxation \cref{gatedconvex_simpler}. Suppose we sampled
$\Tilde{P} \geq 8\kappa \log(\frac{n}{\delta})$ hyperplane arrangement patterns. With probability at least $1 - \delta$, 
$$
\Tilde{p}_1^{*} \leq \sqrt{2}\beta \frac{\lVert y \rVert_2}{\sqrt{\lambda_{min}(\mathcal{M})}},
$$
provided that $\mathcal{M}$ is invertible.
\end{proposition}
Then, we lower bound $p_1^{*}$ using the dual problem of \cref{gatedconvex_simpler}. As the dual problem becomes a maximization of the dual objective with respect to a dual variable $\lambda$, choosing $\lambda$ well can give a meaningful lower bound on $p_1^{*}$. In the proof we choose a scalar multiple of $y$ to obtain a bound that has a similar scale with $\sqrt{2}\beta \frac{\lVert y \rVert_2}{\sqrt{\lambda_{min}(\mathcal{M})}}$.
\begin{proposition}
\label{prop4}
Let $p_1^{*}$ be the optimal value of the unconstrained problem using all hyperplane arrangement patterns. When we write $\mathcal{M}_i = D_iXX^{T}D_i$, 
$$
G\beta \frac{\lVert y \rVert_2}{\sqrt{\lambda_{\max}(XX^{T})}} \leq p_1^{*},
$$
where 
$$
G = 1 - \frac{\beta}{2\max_{i \in [P]}\sqrt{y^{T}\mathcal{M}_{i}y}}.
$$
\end{proposition}

A direct corollary simplifies the lower bound. \cref{simplifiedlowerbound} can be used when $\beta$ is sufficiently smaller than $\lVert y \rVert \sqrt{\lambda_{\min}(\mathcal{M})}$. As $\lVert y \rVert$ is the scale of $O(\sqrt{n})$ and $\lambda_{\min}(\mathcal{M})$ grows with $d$ for random data matrix $X$, it is likely that the value $\lVert y \rVert_2 \sqrt{\lambda_{\min}(\mathcal{M})}$ is much larger than $\beta$ for most cases.
\begin{corollary}
\label{simplifiedlowerbound}
Suppose further that $\lVert y \rVert_2 \sqrt{\lambda_{\min}(\mathcal{M})} \geq \beta$. Then, 
$$
\frac{\beta}{2} \frac{\lVert y \rVert_2}{\sqrt{\lambda_{\max}(XX^{T})}} \leq p_1^{*}.
$$
\end{corollary}
By using \cref{prop3} and \cref{prop4}, we obtain a relative error bound between $p_1^{*}$ and $\Tilde{p}_1^{*}$.
\begin{theorem}
\label{t2}
Let $p_1^{*}$ and $\Tilde{p}_1^{*}$ be optimal values of problem \eqref{gatedconvex_simpler} with all possible hyperplane arrangements and randomly sampled arrangements, respectively. Suppose we sampled $\Tilde{P} \geq 8\kappa \log(\frac{n}{\delta})$ hyperplane arrangement patterns and $\mathcal{M}$ is invertible. We have
$$
0 \leq \Tilde{p}_1^{*} - p_1^{*} \leq \sqrt{2}\beta \frac{\lVert y \rVert_2}{\sqrt{\lambda_{min}(\mathcal{M})}},
$$
and
$$
p_1^{*} \leq \Tilde{p}_1^{*} \leq \frac{\sqrt{2\kappa}}{G} p_1^{*},
$$
with probability at least $1 - \delta$.
\end{theorem}

Note that different from the $l_2$ regularized case, in this case, we have a relative error bound that is independent of the total number of hyperplane arrangement patterns.
%An important implication of \cref{t2} is that we can bound the relative error between the original and the approximate error as a dimensionless constant, regardless of the number of randomly sampled hyperplanes $\Tilde{P}$. Hence, even though we sample only a very small subportion of the whole set of hyperplane arrangement patterns which are $O(r(\frac{n}{r})^r)$ in size, it becomes a good approximation of the original problem with high probability.

A better approximation is possible if we take into account two things: one is that we can use $\max_i \sqrt{y^{T}\mathcal{M}_iy}/\lVert y \rVert_2$ instead of $\lambda_{\max}(XX^T)$ to bound the relative error with a tighter bound, and we can also use $\mathcal{M}_{\Tilde{P}}$ directly instead of $\lambda_{\min}(\mathcal{M})$. Hence, the overall bound can be improved to 
\begin{equation}
\label{tighter}
p_0^{*} \leq p_1^{*} \leq \frac{\sqrt{2}}{G} \frac{\max_i \sqrt{y^{T}\mathcal{M}_iy} \sqrt{y^{T}\mathcal{M_{\Tilde{P}}}^{-1}y}}{\lVert y \rVert_2^2} p_0^{*},
\end{equation}
where we sampled sufficiently many hyperplane arrangement patterns to make $\mathcal{M}_{\Tilde{P}}$ invertible. 
\subsection{Connection to the MAX-CUT Problem}
\label{3.4. connection to maxcut}
There is an interesting connection between the upper bound of the relative error and the MAX-CUT problem. From bound (\ref{tighter}), we can bound 
$$
\max_i \sqrt{y^{T}\mathcal{M}_iy} \leq \sqrt{\max_{b \in \{0,1\}^{n}} b^{T} diag(y)XX^{T}diag(y) b},
$$
where solving the right-hand side is equivalent to solving the max-cut problem of the graph with an adjacency matrix 
$$
\frac{1}{4} \begin{bmatrix} I \\ 1^{T} \end{bmatrix} diag(y)XX^{T}diag(y) \begin{bmatrix} I & 1 \end{bmatrix}.
$$
Hence, in special cases where $X$ is orthogonally separable or when $X$ has negative correlation, we may further bound the relative error. The connection with MAX-CUT shows that the relative error, or the possibility of approximation, is intrinsically related to the structure of the dataset and its clusterization properties.
\subsection{Scale of $\kappa$ for Random Data}
\label{3.5. kappa}
%From \cref{3.1.Noreg} to \cref{3.3. group l1reg}, we could see that $\kappa$ is related to (1) the number of hyperplane arrangement patterns one should sample for good approximation, and (2) the relative approximation error. Hence, upper bounding $\kappa$ is important to understand the quality of these approximation theorems. 

In this subsection, we show that under (A1), we can show that $\kappa \leq O(c)$ holds with high probability for sufficiently large $d$. This means that for random data, as long as $d$ grows with $n$, we only need to sample $O(\log n)$ hyperplane arrangement patterns to obtain a constant factor approximation of the unconstrained relaxation. This bound is more practical compared to most bounds that require networks of width at least $n$. 

To prove the identity, we first use the non-asymptotic inequality deduced from Gordon's minimax comparison 
%with high probability, $\lambda_{\max}(XX^{T})$ is upper bounded with an $O(n)$ factor.
%\begin{proposition}
%\label{MP}
%Suppose $X$ denotes a random matrix of size $n \times d$ whose entries are independent, identically distributed with mean 0 and variance 1. Then for every $t > 0$, we have
%$$
%\mathbb{P}(\sqrt{\lambda_{\max}(XX^{T})} \geq \sqrt{n}+\sqrt{d}+t) \leq 2\exp(-ct^2).
%$$
%\end{proposition}
%Hence, with high probability, we know that 
$$\lambda_{max}(XX^{T}) \leq 2 (\sqrt{n} + \sqrt{d})^2,$$ 
which holds with high probability. To upper bound $\kappa$, finding a lower bound of $\lambda_{\min}(\mathcal{M})$ is enough.

To lower bound $\lambda_{\min}(\mathcal{M})$, we find a closed form expression of $\mathcal{M}_{ij}$ and use \cref{EK} to approximate $\mathcal{M}$ via linearization. We need $d$ to be sufficiently large to apply concentration inequalities on $\lVert x_i \rVert_2$, the norm of rows of $X$.

\begin{theorem}
\label{mineig}
Assume (A1) and suppose $c\geq 1, \delta > 0$. There exists $d_1$ such that if $d \geq d_1$, with probability at least $1 - \delta'$, we have
$$
\lambda_{min}(\mathcal{M}) \geq \frac{d}{10}.
$$
\end{theorem}
Hence, $\kappa$ is indeed upper bounded with a constant factor when $d$ is sufficiently large. Using the fact directly leads to the following corollary.
\begin{corollary}
\label{c3.10}
When the conditions of \cref{mineig} holds, with probability at least $1 - \delta' - e^{-Cn}$, we have
$$
\kappa \leq 20(\sqrt{c}+1)^2,
$$
for some $C > 0$. Moreover, let $p_1^{*}$ and $\Tilde{p}_1^{*}$ be optimal values of problem \eqref{gatedconvex_simpler} with all possible hyperplane arrangements and randomly sampled arrangements, respectively.  When we sample $\Tilde{P} \geq 160(\sqrt{c}+1)^2 \log(\frac{n}{\delta})$ hyperplane arrangement patterns, we have
$$
p_1^{*} \leq \Tilde{p}_1^{*} \leq \frac{2\sqrt{10}}{G}(\sqrt{c}+1)p_1^{*}
$$
with probability at least $1 - \delta - \delta' - e^{-Cn}.$
\end{corollary}
\section{Extension to the Constrained Problem}
\label{CR}
In this section, we move a step further by discussing the relative error between the optimal value of the convex reformulation and its random relaxation. Let $\{w_i^{*}\}_{i=1}^{m}$ be the solution of the unconstrained relaxation of the original problem and let $u_i^{*}, v_i^{*} \in \mathcal{K}_{D_i}$ satisfying $w_i^{*} = u_i^{*} - v_i^{*}$ with minimal $\lVert u_i^{*} \rVert_2 + \lVert v_i^{*} \rVert_2$. 
%We follow a similar strategy from \cite{mishkin2022fast}, by first solving the gated problem (\ref{gatedconvex_simpler}) and decomposing the solutions $w_i^{*}$ of the gated problem with 
Although there is no universal constant $C$ that satisfies 
$$
\lVert u_i^{*} \rVert_2 + \lVert v_i^{*} \rVert_2 \leq C\lVert w_i^{*} \rVert_2,
$$
for any cone $\mathcal{K}$ \cite{mishkin2022fast}, we can show that for random data $X$, it is likely that there exists reasonably large $\mathcal{C}$ that satisfies
$$
\lVert u_i^{*} \rVert_2 + \lVert v_i^{*} \rVert_2 \leq \mathcal{C}\lVert w_i^{*} \rVert_2,
$$
for randomly sampled cones $\mathcal{K}_{\Tilde{D}_i}$ and all $w_i^{*}$ with high probability. In 2 dimensions, $\mathcal{C}$ is directly related to the angle between two rays of the convex cone. Hence, $\mathcal{C}$ is analogous to the `sharpness of each cone $\mathcal{K}_{\Tilde{D}_i}$'. We defer the proofs to \cref{proof4}.
\subsection{Cone Sharpness $\mathcal{C}$}
We start by defining the cone sharpness constant $\mathcal{C}(\mathcal{K}, z)$ for a given convex cone $\mathcal{K}$.
\begin{definition}
For a cone $\mathcal{K}$ and unit vector $z$, the sharpness with respect to $z$ is defined as
$$
C(\mathcal{K}, z) = \min_{u, v \in \mathcal{K}, u-v = z} \lVert u \rVert_2 + \lVert v \rVert_2.
$$
\end{definition}
We can upper bound the unconstrained relaxation with the constrained relaxation using the cone sharpness. 
\begin{proposition}
\label{coneapprox}
Let $\Tilde{p}_0^{*}$ and $\Tilde{p}_1^{*}$ be optimal values of the Gaussian relaxation of the convex reformulation and its unconstrained version, respectively. Suppose the unconstrained problem has solutions $w_i^{*}$ for $i \in [\Tilde{P}]$, and let 
$$
C = \max_{i \in [\Tilde{P}]} \mathcal{C}(\mathcal{K}_{\Tilde{D}_i}, \frac{w_i^{*}}{\lVert w_i^{*} \rVert_2}).
$$
Then, $\Tilde{p}_0^{*} \leq C\Tilde{p}_1^{*}$ holds.
\end{proposition}

Now, we introduce a strategy to upper bound $\mathcal{C}(\mathcal{K}_{\Tilde{D}_i}, z)$. An interesting fact is that when (A1) is satisfied, the cone sharpness constant is not large with high probability, and can be bounded by $O(\sqrt{\log n})$ factor. The intuition here is that when $n \asymp d$, we have approximately $O((n/d)^d)$ many cones - which is only exponential in $d$ when $n/d = c$. Hence, in $d$ dimensions, the number of cone constraints is not that many, and we can have reasonable upper bounds on the sharpness.

We first start with a simple proposition that upper bounds $C(\mathcal{K}_{\Tilde{D}_i}, z)$. The idea is similar to Chebyshev centers of a polyhedron \cite{boyd2004convex}. 
\begin{proposition}
\label{chebyshev_sim}
Take any unit vector $z$, and suppose $\mathcal{K} = \{u | (2D - I) Xu \geq 0\}$. If there exists vector $u$ that satisfies
$$\lVert u \rVert_2 \leq 1, \quad (2D - I)Xu \geq \epsilon \cdot |(2D - I)Xz|,$$
we know that
$$
C(\mathcal{K}, z) \leq 1 + \frac{1}{\epsilon}.
$$
\end{proposition}
When (A1) is satisfied and for a randomly sampled cone $\mathcal{K}_{\Tilde{D}_i}$, with the rotational invariance of $X$, we can rotate the cone to contain $e_1$ and the distribution of other elements on $X$ will not change. Therefore, without loss of generality, we may assume that $\mathcal{K} = \{u \ |\  \Tilde{X}u \geq 0\},$ where the first column of $\Tilde{X}$ is the absolute value sampled from $\mathcal{N}(0,1)$ and the other columns are sampled from $\mathcal{N}(0,1)$. 

We wish to construct a vector $u$ that satisfies the conditions in \cref{chebyshev_sim}. Using a novel 
application of Gordon's comparison, we find such a vector for any unit vector $z$.
\begin{theorem}
\label{uniformexistance}
Let $b \in \mathbb{R}^{n}$ sampled from the folded normal distribution, and $X \in \mathbb{R}^{n \times d}$ be a matrix where each entries are sampled from a normal distribution. Consider the random variable
$$
F(X, b) = \max_{\lVert z \rVert_2 = 1} \min_{\substack{Xu \geq -Xz - kb\\Xu \geq Xz - kb\\k \geq 0}} \lVert u \rVert_2 + k,
$$
where $u, z \in \mathbb{R}^{d}$. Then, with probability at least $1 - 1/n^{10} - e^{-Cd}$ for some positive constant $C$, 
\begin{equation}
\label{superimportantconeconstraint}
F(X,b) \leq 200c\sqrt{c\log 2n}.
\end{equation}
\end{theorem}
%\label{existancelambda}
%Suppose $X$ is a $d \times d$ random matrix with entries i.i.d. and sampled from $\mathcal{N}(0,1)$. Also, suppose $b$ is an arbitrary vector that satisfies $\lVert b \rVert_2 \leq 2\sqrt{n}$ and $n/d = c$. Then, for some $B > 0$, with probability at least $1 - 8\exp(-Bn)$, there exists a vector that satisfies $\lVert \lambda \rVert_2 \leq 5c$ and $X\lambda \geq b$. 
%\end{lemma}
%\begin{remark}
%In practice, we can see that there exists a vector satisfying $\lVert \lambda \rVert_2 \leq 3$ and $X\lambda \geq b$ for moderate $n$.
%\begin{figure}[htbp]
 %   \centering
  %  \includegraphics[width=0.4\textwidth]{figures/d=500.png}
   % \caption{Numerical illustration of \cref{existancelambda}, $n = 500$}
    %\label{figure1}
%\end{figure}
%\end{remark}
For any given unit vector $z$, we can construct $u$ for the rotated cone $\mathcal{K} = \{u\ |\ \Tilde{X}u \geq 0\}$ by solving the minimization problem in \cref{uniformexistance}. Eventually, we can bound $\mathcal{C}(\mathcal{K},z)$ with a logarithmic factor.
\begin{corollary}
\label{conesharp}
Suppose $n$, $d$ are sufficiently large that \cref{superimportantconeconstraint} holds with probability at least $1 - \delta''$, for $b \in \mathbb{R}^{n}$ sampled from a folded normal distribution and $X \in \mathbb{R}^{n \times {d-1}}$ sampled from a normal distribution. Then, with probability at least $1 - \delta''$,
$$
C(\mathcal{K}_{\Tilde{D}_i}, z) \leq 2 + 200c\sqrt{c \log 2n},
$$
also holds for all unit vectors $z$.
\end{corollary}
A direct corollary is that we can approximate the convex reformulation with its unconstrained version, having $O(\sqrt{\log n})$ scale relative bound (\cref{c3}).
\subsection{Proof of the Main Results}
After the bound on the cone sharpness, the proof of the main theorems follows almost immediately. By considering the convex reformulation of the original problem \eqref{nonconvex_original}, and first upper bounding it with the unconstrained problem, then upper bounding it again with the unconstrained problem with random relaxation, we can find a relative optimality gap between the original problem and its convex relaxation. It is clear that we immediately get a tractable polynomial-time randomized algorithm by solving the randomly relaxed convex problem. At last, by identifying stationary points of the original problem as the global minimum of randomly subsampled convex problems, we obtain \cref{losslandscape}. See \cref{Proofmain} for a detailed proof.

\section{Conclusion}
\label{Conclu}
In this paper, we provided guarantees of approximating the equivalent convex program given in \cite{pilanci2020neural} with a much smaller random subprogram. With assumptions on $X$ and the dimension $d$, we proved that the optimal value of the subsampled convex program approximates the full convex program up to a logarithmic factor. Using the approximation results we discuss novel insights on training two-layer neural networks, by showing that under mild assumptions local gradient methods converge to stationary points with optimality guarantees, and propose a practical algorithm to train neural networks with guarantees. 

We hope to improve the work in two ways: First, removing the logarithmic factor of the approximation would be an important problem to tackle. Also, extending the theorems to different architectures, i.e. CNNs, transformers, and multi-layer networks would be meaningful.

\section*{Acknowledgements}
This work was supported in part by the National Science Foundation (NSF) under Grant DMS-2134248; in part by the NSF CAREER Award under Grant CCF-2236829; in part by the U.S. Army Research Office Early Career Award under Grant W911NF-21-1-0242; and in part by the Office of Naval Research under Grant N00014-24-1-2164.

\section*{Impact statement}
Our work advances the understanding training of neural networks, and we feel no societal consequences should be specifically highlighted within the paper.

\bibliography{Newton}
\bibliographystyle{icml2024}

%%%%%%%%%%%%%%%%%%%%%%%%%%%%%%%%%%%%%%%%%%%%%%%%%%%%%%%%%%%%%%%%%%%%%%%%%%%%%%%
%%%%%%%%%%%%%%%%%%%%%%%%%%%%%%%%%%%%%%%%%%%%%%%%%%%%%%%%%%%%%%%%%%%%%%%%%%%%%%%
% APPENDIX
%%%%%%%%%%%%%%%%%%%%%%%%%%%%%%%%%%%%%%%%%%%%%%%%%%%%%%%%%%%%%%%%%%%%%%%%%%%%%%%
%%%%%%%%%%%%%%%%%%%%%%%%%%%%%%%%%%%%%%%%%%%%%%%%%%%%%%%%%%%%%%%%%%%%%%%%%%%%%%%
\newpage
\appendix
\onecolumn
\section{Useful Lemmas}
\label{usefullemma}
\begin{lemma}
\label{matrixchernoff}
(Matrix Chernoff \cite{tropp2015introduction}, \cite{harveylecture} ) Let $X_1, X_2, ..., X_k$ be independent, random, symmetric real $n \times n$ matrix with $0 \preccurlyeq X_i \preccurlyeq RI$. Let $\mu_{min} I \preccurlyeq \sum_{i=1}^{k} \mathbb{E}[X_i] \preccurlyeq \mu_{max} I$. Then, for all $\delta \in [0,1]$, 
$$
\mathbb{P}(\lambda_{max}(\sum_{i=1}^{k} X_i) \geq (1+\delta)\mu_{max}) \leq ne^{-\delta^{2}\mu_{max}/3R} 
$$
and
$$
\mathbb{P}(\lambda_{min}(\sum_{i=1}^{k} X_i) \leq (1-\delta)\mu_{min}) \leq ne^{-\delta^{2}\mu_{min}/2R} 
$$
holds.
\end{lemma}
\begin{lemma}
\label{Gordon}
(Gordon's comparison \cite{thrampoulidis2014gaussian}) Let $\Phi(X)$ be
$$
\Phi(X) = \min_{x \in S_x} \max_{y \in S_y} y^TXx + \psi(x,y),   
$$
where $X$ is a random matrix with i.i.d. Gaussian entries, and $\phi(g,h)$ be
$$
\phi(g,h) = \min_{x \in S_x} \max_{y \in S_y} \lVert x \rVert_2 g^Ty + \lVert y \rVert_2 h^Tx + \psi(x,y),
$$
where entries of $g,h$ are i.i.d. and sampled from $\mathcal{N}(0,1)$, and $S_x$ and $S_y$ are compact, $\psi$ is continuous. Then, we have
$$
\mathbb{P}(\Phi(X) \leq M) \leq 2 \mathbb{P}(\phi(g,h) \leq M).
$$
\end{lemma}
\begin{lemma}
\label{Weyl}
(Weyl's theorem on eigenvalues) Suppose $A, B \in \mathbb{R}^{n \times n}$ are symmetric. Then, the below estimates on the minimum eigenvalue of $A, B$ hold.
\begin{equation}
\label{Weyl1}
    \lambda_{min}(A) \geq \lambda_{min} (B) - \lVert A - B \rVert_{F},
\end{equation}
\begin{equation}
\label{Weyl2}
    \lambda_{min}(A+B) \geq \lambda_{min}(A) + \lambda_{min}(B).
\end{equation}
Here, $\rho(A)$ denotes the spectral radius of $A$, and $\lVert A \rVert_F$ is the Frobenius norm of $A$.
\end{lemma}
\begin{lemma}
\label{EK}
(Spectrum of kernel random matrices \cite{el2010spectrum}) Suppose we obtain $n$ i.i.d. vectors $x_i$ from $\mathcal{N}(0, I_d)$.  Let's consider the kernel matrix
$$
K_{i,j} = f(\frac{x_i \cdot x_j}{d}).
$$
We assume that: \\
(a) $n/d, d/n$ are bounded as $d \rightarrow \infty$.\\
(b) $f$ is a $C^1$ function in a neighborhood of 1, and a $C^3$ function in a neighborhood of 0.\\
Under the assumptions, the kernel matrix $M$ can (in probability) be approximated consistently in operator norm, when $d, n \rightarrow \infty$, by the matrix $K'$, where
$$
K' = (f(0) + \frac{f''(0)}{2d}) 11^{T} + \frac{f'(0)}{d}XX^T
+ (f(1)-f(0)-f'(0)) I_n.
$$
\end{lemma}

\newpage
\section{Proofs in Section 3.}
\label{proof3}
\begin{proposition}(Proposition 3.1. of the paper)
Suppose no two rows of $X$ are parallel. There exists hyperplane patterns $\Tilde{D}_1$, $\Tilde{D}_2$, ..., $\Tilde{D}_{2n}$ such that for all $y \in \mathbb{R}^{n}$, there exists $u_1, u_2, ... u_{2n} \in \mathbb{R}^{d}$ satisfying
\begin{equation*}
\sum_{i=1}^{2n} \Tilde{D}_iXu_i = y.
\end{equation*}
\end{proposition}
\begin{proof}
Let's begin with a simple observation: for any $i \in [n]$, there exists two different hyperplane pattern $D_{1,i}$ and $D_{2,i}$ satisfying
\begin{equation*}
|D_{1,i} - D_{2,i}| = diag(e_i).
\end{equation*}
The construction is intuitive when thought geometrically, passing through only the $i$th plane. A rigorous construction is as below:
First, choose a point $c \in \mathbb{R}^{d}$ satisfying\\
\begin{equation*}
c \cdot X_i = 0 ,\quad  c \cdot X_j \neq 0 \quad for \ all \quad j \neq i,
\end{equation*}
where $X_1, X_2, ... ,X_n$ are rows of $X$. We can choose such point because the set $\{u|u \cdot X_i = 0\} \cap \{u|u \cdot X_j = 0\}$ has measure 0 in $\{u|u \cdot X_i = 0\}$, provided no two rows of $X$ are parellel. Now, we choose $\epsilon$ to be small such that
\begin{equation*}
    \epsilon = \min_{j \neq i, X_j \cdot X_i \neq 0}\frac{|c \cdot X_j|}{2|X_i \cdot X_j|}\ > 0.
\end{equation*}
and take two points $c + \epsilon X_i$, $c - \epsilon X_i$. Now, for the $i$th hyperplane, 
\begin{equation*}
    (c + \epsilon X_i) \cdot X_i > 0, \quad 
    (c - \epsilon X_i) \cdot X_i < 0,
\end{equation*}
and for all other hyperplanes $j \neq i$, 
\begin{equation*}
    \epsilon |X_i \cdot X_j| \leq \frac{1}{2} |c \cdot X_j|,
\end{equation*}
thus the sign of $(c + \epsilon X_i) \cdot X_j, c \cdot X_j, (c - \epsilon X_i) \cdot X_j$ are identical.
This means the hyperplane patterns of $c + \epsilon X_i$ and $c - \epsilon X_i$ differ only for the $i$th plane, meaning we have found two planes $D_{1,i}, D_{2,i}$ such that $|D_{1,i} - D_{2,i}| = diag(e_i).$
Now, define $\Tilde{D}_{2i-1} = D_{1,i}, \Tilde{D}_{2i} = D_{2,i}$ for $i \in [n]$, and choose $u_i$ vectors such that
\begin{equation*}
    (Xu_i)[i,1] = y[i,1]
\end{equation*}
where the notation $A[i,j]$ denotes the $i$th row, $j$th column of $A$. We can always choose such a vector, as we can choose arbitrary vector $v$ and scale $c_i$ until it matches $c_i(Xv)[i,1] = y[i,1]$. Hence, we have found such $\Tilde{D}$s that we proposed earlier.
\end{proof}
\begin{proposition}(Proposition 3.2. of the paper)
Suppose we sampled 
$$
\Tilde{P} = 2\kappa \log(\frac{n}{\delta})
$$
hyperplane arrangement patterns, provided that $\mathcal{M}$ is invertible. Then, with probability at least $1-\delta$, there exists $u_1, u_2, ..., u_{\Tilde{P}}$ satisfying
$$
\sum_{i=1}^{\Tilde{P}} \Tilde{D}_iXu_i = y.
$$
\end{proposition}
\begin{proof}
Think of $X_i = \Tilde{D}_iXX^{T}\Tilde{D}_i$ as a random independent series of symmetric real $n \times n$ matrices. Also, we know that for any vector $u \in \mathbb{R}^{n}$, 
$$
\sqrt{u^{T}X_iu} \leq \sqrt{\lambda_{max}(XX^{T})} \lVert \Tilde{D}_iu \rVert_2 \leq \sqrt{\lambda_{max}(XX^{T})} \lVert u \rVert_2,
$$
hence we know that for all $X_i$s, $X_i \preccurlyeq \lambda_{max}(XX^{T}) I$. Now, from \cref{matrixchernoff}, take $\delta = 1$. Then, we obtain 
$$
\mathbb{P}(\lambda_{min}(\sum_{i=1}^{\Tilde{P}} X_i) \leq 0) \leq ne^{-\mu_{min}/2\lambda_{max}(XX^{T})} 
$$
where $\mu_{min} = \Tilde{P} \lambda_{min}(\mathcal{M})$. Hence, when we plug in 
$$
\Tilde{P} = 2\kappa \log(\frac{n}{\delta}),
$$
we can see that $\mathbb{P}(\lambda_{min}(\sum_{i=1}^{\Tilde{P}} \Tilde{D}_iXX^{T}\Tilde{D}_i) \leq 0) \leq \delta$, and $\sum_{i=1}^{\Tilde{P}} \Tilde{D}_iXX^{T}\Tilde{D}_i$ is invertible with probability $1 - \delta$.

At last, consider the augmented matrix $\mathcal{X} = [\Tilde{D}_1X | \Tilde{D}_2X | ... | \Tilde{D}_{\Tilde{P}}X]$, which is a $n \times \Tilde{P}d$ matrix. As $\mathcal{X}\mathcal{X}^{T} = \sum_{i=1}^{\Tilde{P}} \Tilde{D}_iXX^{T}\Tilde{D}_i$ is invertible, we can see that $\mathcal{X}$ has $n$ nonzero singular values, hence has rank $n$. This means that the column space of $[\Tilde{D}_1X | \Tilde{D}_2X | ... | \Tilde{D}_{\Tilde{P}}X]$ also has rank $n$, and we can find vectors $u_1, u_2, ..., u_{\Tilde{P}}$ that satisfies 
$$
\sum_{i=1}^{\Tilde{P}} \Tilde{D}_iXu_i = y
$$
for any given $y \in \mathbb{R}^{n}$.
\end{proof}
The two propositions directly lead to the following corollary.
\begin{corollary}(Corollary 3.3. of the paper)
Suppose $p_0^{*}$ and $p_1^{*}$ are solutions to the unconstrained problem that uses all possible hyperplane arrangement patterns and its randomized relaxation, respectively, and the regularization $\beta = 0$. Also, suppose $\Tilde{P} \geq 2\kappa\log(\frac{n}{\delta})$. Then, $p_0^{*} = p_1^{*} = 0$ with probability at least $1 - \delta$.
\end{corollary}
\begin{proof}
From \cref{prop1}, we can find $v_1, v_2, ..., v_{2n}$ that satisfies 
$$
\sum_{i=1}^{2n} D'_iXv_i = y
$$
for specific $D'_i$s. Hence, for the unconstrained problem with all possible hyperplane arrangements, we can choose $u_i$s to be $u_i = v_i$ if $D_i = D'_i$, $u_i = 0$ otherwise to perfectly fit a given vector $y \in \mathbb{R}^{n}$. Also, from \cref{prop2}, we can fit a given vector $y$ with probabillity at least $1 - \delta$ when we sample hyperplane arrangement patterns more than $2\kappa \log(\frac{n}{\delta})$ times. Hence, with probability at least $1 - \delta$, we know that $p_1^{*} = 0$. 
\end{proof}
\begin{theorem}(Theorem 3.4. of the paper)
Suppose
\begin{equation}
\label{l2reg}
p^{*} = \min_{u_i \in \mathbb{R}^{d}}\frac{1}{2} \lVert \sum_{i=1}^{\Tilde{P}} \Tilde{D}_iXu_i - y\rVert_2^2 + \frac{\beta}{2} \sum_{i=1}^{\Tilde{P}} \lVert u_i \rVert_2^{2}.
\end{equation}
Then, with probability at least $1-\delta$ there exists scalars $C_1, C_2 > 0$ that satisfies
$$
\frac{C_1}{\Tilde{P}} \leq p^{*} \leq \frac{C_2}{\Tilde{P}},
$$
provided that $\Tilde{P} \geq 12\kappa \log(\frac{2n}{\delta})$ and $\mathcal{M}$ is invertible.
\end{theorem}
\begin{proof}
Solving the $L_2$ regularized problem (\ref{l2reg}) is equivalent to solving
\begin{equation}
\label{constrained}
\min_{u_i \in \mathbb{R}^{d}, w} \frac{1}{2} \lVert w-y \rVert_2^2 + \frac{\beta}{2}\sum_{i=1}^{\Tilde{P}} \lVert u_i \rVert_2^2.
\end{equation}
subject to
$$
w = \sum_{i=1}^{\Tilde{P}} D_iXu_i.
$$
The lagrangian of problem (\ref{constrained}) becomes
$$
L(u, w, \lambda) = \frac{1}{2} \lVert w-y \rVert_2^2 + \lambda^{T}(w - \sum_{i=1}^{\Tilde{P}} D_iXu_i) + \frac{\beta}{2} \sum_{i=1}^{\Tilde{P}} \lVert u_i \rVert_2^2.
$$
As the constraint is linear equality and the objective is convex with respect to the arguments, strong duality holds and 
$$
p^{*} = \min_{u_i \in \mathbb{R}^{d}, w \in \mathbb{R}^{n}} \max_{\lambda \in \mathbb{R}^{n}} L(u_i, w, \lambda) = \max_{\lambda \in \mathbb{R}^{n}}\min_{u_i \in \mathbb{R}^{d}, w \in \mathbb{R}^{n}} L(u_i, w, \lambda).
$$
For a given $\lambda$, we can optimize for $u_i$ and $w$ to minimize $L(u_i, w, \lambda)$, where we get $w = y - \lambda$ and $u_i = \frac{1}{\beta} X^{T}\Tilde{D}_i\lambda$. Substituting leads to 
$$
p^{*} = \max_{\lambda \in \mathbb{R}^{n}} -\frac{1}{2} \lVert \lambda \rVert_2^2 + \lambda^{T}y - \frac{1}{2\beta} \lambda^{T} \big(\sum_{i=1}^{\Tilde{P}} \Tilde{D}_iXX^{T}\Tilde{D}_i\big) \lambda.
$$
Now, let's write $\frac{1}{\Tilde{P}} \sum_{i=1}^{\Tilde{P}} \Tilde{D}_iXX^{T}\Tilde{D}_i = \mathcal{M}_{\Tilde{P}}$. Then, the dual problem can be simplified as
$$
p^{*} = \max_{\lambda \in \mathbb{R}^{n}} -\frac{1}{2\beta}\lambda^{T}(\beta I+ \Tilde{P}\mathcal{M}_{\Tilde{P}})\lambda + \lambda^{T}y
$$
and the optimum $p^{*}$ is given as
$$
p^{*} = \frac{\beta}{2} y^{T}(\beta I+ \Tilde{P}\mathcal{M}_{\Tilde{P}})^{-1} y.
$$
When $\Tilde{P} \geq 12\kappa \log(\frac{2n}{\delta})$, by \cref{matrixchernoff} we can see that
$$
\mathbb{P}\bigg(\lambda_{min}(\mathcal{M}_{\Tilde{P}}) \geq \frac{\lambda_{min}(\mathcal{M})}{2}\bigg) \geq 1 - \frac{\delta}{2}
$$
and
$$
\mathbb{P}\bigg(\lambda_{max}(\mathcal{M}_{\Tilde{P}}) \leq \frac{3\lambda_{max}(\mathcal{M})}{2}\bigg) \geq 1 - \frac{\delta}{2}.
$$
Hence, with probability at least $1 - \delta$, we can see that
$$
(\beta + \frac{\Tilde{P}}{2}\lambda_{\min}(\mathcal{M})) I \preccurlyeq \Tilde{P}\mathcal{M}_{\Tilde{P}}+\beta I \preccurlyeq (\beta + \frac{3\Tilde{P}}{2}\lambda_{\max}(\mathcal{M})) I
$$
and we can find $\mathcal{R}_1, \mathcal{R}_2 > 0$ that satisfies
$$
\mathcal{R}_1\Tilde{P} I \preccurlyeq \Tilde{P}\mathcal{M}_{\Tilde{P}}+\beta I \preccurlyeq \mathcal{R}_2\Tilde{P} I.
$$
An example is $\mathcal{R}_1 = \frac{\lambda_{\min}(\mathcal{M})}{2}$, $\mathcal{R}_2 =  \frac{3\lambda_{\min}(\mathcal{M})}{2}+1$, provided that $\Tilde{P}$ is larger than $\beta$. This means that 
$$
\frac{\beta\lVert y \rVert_2^2}{2\mathcal{R}_2 \Tilde{P}} \leq p^{*} \leq \frac{\beta\lVert y \rVert_2^2}{2\mathcal{R}_1 \Tilde{P}},
$$
and take $C_1 = \frac{\beta\lVert y \rVert_2^2}{2\mathcal{R}_2}, C_2 = \frac{\beta\lVert y \rVert_2^2}{2\mathcal{R}_1}$ to finish the proof.
\end{proof}
\begin{proposition}(Proposition 3.5. of the paper)
Suppose
$$
p_{1}^{*} = \min_{u_i \in \mathbb{R}^{d}} \frac{1}{2}\lVert \sum_{i=1}^{\Tilde{P}} \Tilde{D}_iXu_i - y \rVert_2^2 + \beta \sum_{i=1}^{\Tilde{P}}\lVert u_i \rVert_2,
$$
and $\Tilde{P} \geq 8\kappa\log(\frac{n}{\delta})$. Then with probability at least $1 - \delta$, 
$$
p_1^{*} \leq \sqrt{2}\beta \frac{\lVert y \rVert_2}{\sqrt{\lambda_{min}(\mathcal{M})}},
$$
provided that $\mathcal{M}$ is invertible.
\end{proposition}
\begin{proof}
We know that
\begin{align*}
p_1^{*} &= \min_{u_i \in \mathbb{R}^{d}} \frac{1}{2}\lVert \sum_{i=1}^{\Tilde{P}} \Tilde{D}_iXu_i - y \rVert_2 ^2 + \beta \sum_{i=1}^{\Tilde{P}}\lVert u_i \rVert_2\\
&=\min_{\kappa_i \in \mathbb{R}, u_i \in \mathbb{R}^{d}} \frac{1}{2}\lVert \sum_{i=1}^{\Tilde{P}} \Tilde{D}_iXu_i - y \rVert_2 ^2 + \frac{\beta}{2} \sum_{i=1}^{\Tilde{P}} (\kappa_i \lVert u_i \rVert_2^2 + \frac{1}{\kappa_i})\\
&= \min_{\kappa_i \in \mathbb{R}} \min_{u_i \in \mathbb{R}^{d}} \frac{1}{2}\lVert \sum_{i=1}^{\Tilde{P}} \Tilde{D}_iXu_i - y \rVert_2 ^2 + \frac{\beta}{2} \sum_{i=1}^{\Tilde{P}} (\kappa_i \lVert u_i \rVert_2^2 + \frac{1}{\kappa_i}).
\end{align*}
With the same idea from the proof of \cref{l2reg}, we can see that when $\kappa_i$s are given, the dual problem becomes
$$
\max_{\lambda \in \mathbb{R}^{n}} -\frac{1}{2\beta} \lambda^{T} (\beta I + \sum_{i=1}^{\Tilde{P}} \frac{1}{\kappa_i}\Tilde{D}_iXX^{T}\Tilde{D}_i) \lambda + \lambda y,
$$
and the inner minimization problem has the minimum
$$
\min_{u_i \in \mathbb{R}^{d}} \frac{1}{2}\lVert \sum_{i=1}^{\Tilde{P}} \Tilde{D}_iXu_i - y \rVert_2 ^2 + \frac{\beta}{2} \sum_{i=1}^{\Tilde{P}} (\kappa_i \lVert u_i \rVert_2^2 + \frac{1}{\kappa_i}) = \frac{\beta}{2} \sum_{i=1}^{\Tilde{P}} \frac{1}{\kappa_i} + \frac{\beta}{2} y^{T} (\beta I + \sum_{i=1}^{\Tilde{P}} \frac{1}{\kappa_i}\Tilde{D}_iXX^{T}\Tilde{D}_i)^{-1} y.
$$
Hence, we can see that
\begin{align*}
p_1^{*} &= \min_{\kappa_i \in \mathbb{R}} \frac{\beta}{2} \sum_{i=1}^{\Tilde{P}} \frac{1}{\kappa_i} + \frac{\beta}{2} y^{T} (\beta I + \sum_{i=1}^{\Tilde{P}} \frac{1}{\kappa_i}\Tilde{D}_iXX^{T}\Tilde{D}_i)^{-1} y\\
&\leq \frac{\beta}{2} \min_{\kappa \in \mathbb{R}} y^{T}(\beta I + \frac{\kappa}{\Tilde{P}} \sum_{i=1}^{\Tilde{P}} \Tilde{D}_iXX^{T}\Tilde{D}_i)^{-1}y + \kappa\\
&\leq \frac{\beta}{2} \min_{\kappa \in \mathbb{R}} \frac{y^{T}( \frac{1}{\Tilde{P}} \sum_{i=1}^{\Tilde{P}} \Tilde{D}_iXX^{T}\Tilde{D}_i)^{-1}y}{\kappa} + \kappa\\
&\leq \beta \sqrt{y^{T}( \frac{1}{\Tilde{P}} \sum_{i=1}^{\Tilde{P}} D_iXX^{T}D_i)^{-1}y}\\
&\leq \sqrt{2}\beta \frac{\lVert y \rVert_2}{\lambda_{\min}(\mathcal{M})}
\end{align*}

Where the last inequality follows from matrix Chernoff that with probability at least $1 - \delta$, the minimum eigenvalue 
$$
\lambda_{\min}(\frac{1}{\Tilde{P}} \sum_{i=1}^{\Tilde{P}} \Tilde{D}_iXX^{T}\Tilde{D}_i) \geq \frac{\lambda_{\min}(\mathcal{M})}{2}.
$$
\end{proof}
\begin{proposition}(Proposition 3.6. of the paper)
Suppose $$p_{0}^{*} = \min_{u_i \in \mathbb{R}^{d}} \frac{1}{2}\lVert \sum_{i=1}^{P} D_iXu_i - y \rVert_2^2 + \beta \sum_{i=1}^{P}\lVert u_i \rVert_2.$$ Here $D_i$s are all possible hyperplane arrangement patterns. Also, write $\mathcal{M}_i = D_iXX^{T}D_i$. Then, 
$$
G\beta \frac{\lVert y \rVert_2}{\sqrt{\lambda_{\max}(XX^{T})}} \leq p_0^{*},
$$
where 
$$
G = 1 - \frac{\beta}{2\max_{i \in [P]}\sqrt{y^{T}\mathcal{M}_{i}y}}
$$
\end{proposition}
\begin{proof}
Let's think of the dual problem of the original optimization problem. The lagrangian $L(u_i, w, \lambda)$ is given as
$$
L(u_i, w, \lambda) = \frac{1}{2} \lVert w - y \rVert_2^2 + \beta \sum_{i=1}^{P} \lVert u_i \rVert_2 + \lambda^{T}(w - \sum_{i=1}^{P} D_iXu_i),
$$
and the dual problem is $\max_{\lambda}\min_{u_i, w} L(u_i, w, \lambda)$. As there is only a linear equality constraint and strong duality holds, we can see that 
$$
p_0^{*} = \max_{\lambda}\min_{u_i, w} \frac{1}{2} \lVert w - y \rVert_2^2 + \beta \sum_{i=1}^{P} \lVert u_i \rVert_2 + \lambda^{T}(w - \sum_{i=1}^{P} D_iXu_i).
$$
When $\lambda$ is fixed, the inner minimization problem can be solved as:

i) If $\lVert X^{T}D_i\lambda \rVert_2 > \beta$ for some $i$: Take $u_i = C X^{T}D_i\lambda$ and send $C$ to infinity to obtain $-\infty$ as the solution to the inner minimization problem.

ii) If $\lVert X^{T}D_i\lambda \rVert_2 \leq \beta$ for all $i$: Each $u_i = 0$ when the inner minimization problem is solved, due to Cauchy-Schwartz and 
$$
\beta \lVert u_i \rVert_2 \geq \lVert X^{T}D_i\lambda \rVert_2 \lVert u_i \rVert_2 \geq \lambda^{T} D_iXu_i
$$
holds for all $u_i \in \mathbb{R}^{d}$. Also, $w = y - \lambda$ should hold, and the objective becomes maximizing $-\frac{1}{2} \lVert \lambda \rVert_2^2 + \lambda^{T} y$.

From i), ii), we can see that the dual problem becomes
$$
\max_{\lambda} -\frac{1}{2} \lVert \lambda \rVert_2^2 + \lambda^{T} y 
$$
subject to
$$
\lVert X^{T}D_i\lambda \rVert_2 \leq \beta \quad \forall i \in [P].
$$
Now, let's find the maximal scaling coefficient $k$ and the corresponding $\lambda_k = ky$ that meets all the constraints. For that $\lambda$, it is clear that $p_0^{*} \geq -\frac{1}{2}\lVert \lambda_k \rVert_2^2 + \lambda_k^{T} y$.
When we substitute $ky$ to $\lambda$, we get the constraint of $k$ for each $i$: 
$$
k\lVert X^{T}D_iy \rVert_2 \leq \beta, \quad k \leq \frac{\beta}{\sqrt{y^{T}D_iXX^{T}D_iy}}
$$
must hold. When we write $D_iXX^{T}D_i = \mathcal{M}_i$, $k$ should satisfy
$$
k \leq \frac{\beta}{\sqrt{y^{T}\mathcal{M}_i y}}
$$
for all $i \in [\Tilde{P}]$. Choose $k = \frac{\beta}{\max_{i} \sqrt{y^{T} \mathcal{M}_i y}}$. Substituting $ky$ to the dual problem leads
$$
p_0^{*} \geq -\frac{\beta^2\lVert y \rVert_2^2}{2\max_{i} \{y^{T} \mathcal{M}_i y\}}+\frac{\beta\lVert y \rVert_2^2}{\max_i \{\sqrt{y^{T} \mathcal{M}_i y}\}} = G\frac{\beta\lVert y \rVert_2^2}{\max_i \{\sqrt{y^{T} \mathcal{M}_i y}\}}.
$$
At last, we know that $\mathcal{M}_i \preccurlyeq \lambda_{\max}(XX^T) I$ for all $i \in [P]$. Hence, $y^{T}\mathcal{M}_iy \leq \lambda_{\max}(XX^T) \lVert y \rVert_2^2$ for all $i \in [P]$ and 
$$
p_0^{*} \geq G\beta \frac{\lVert y \rVert_2}{\sqrt{\lambda_{\max}(XX^{T})}}.
$$
\end{proof}
\begin{corollary}(Corollary 3.7. of the paper)
Suppose further that $\lVert y \rVert_2 \sqrt{\lambda_{\min}(\mathcal{M})} \geq \beta$. Then, 
$$
\frac{\beta}{2} \frac{\lVert y \rVert_2}{\sqrt{\lambda_{\max}(XX^{T})}} \leq p_1^{*}.
$$
\end{corollary}
\begin{proof}
We know that $G = 1 - \frac{\beta}{2\max_{i \in [P]} \sqrt{y^{T}\mathcal{M}_iy}} \geq 1 - \frac{\beta}{2\lVert y \rVert_2\sqrt{\lambda_{\min}(\mathcal{M})}} \geq \frac{1}{2}.$
\end{proof}
\begin{theorem}(Theorem 3.8. of the paper)
Let $p_1^{*}$ and $\Tilde{p}_1^{*}$ be optimal values of problem \eqref{gatedconvex_simpler} with all possible hyperplane arrangements and randomly sampled arrangements, respectively. Suppose we sampled $\Tilde{P} \geq 8\kappa \log(\frac{n}{\delta})$ hyperplane arrangement patterns and $\mathcal{M}$ is invertible. We have
$$
0 \leq \Tilde{p}_1^{*} - p_1^{*} \leq \sqrt{2}\beta \frac{\lVert y \rVert_2}{\sqrt{\lambda_{min}(\mathcal{M})}},
$$
and
$$
p_1^{*} \leq \Tilde{p}_1^{*} \leq \frac{\sqrt{2\kappa}}{G} p_1^{*},
$$
with probability at least $1 - \delta$.
\end{theorem}
\begin{proof}
The two inequalities directly follow from \cref{prop3} and \cref{prop4}
\end{proof}
\begin{proposition}
$\mathcal{M}_{ij} = \{\frac{1}{2} - \frac{1}{2\pi}\arccos(\frac{x_i \cdot x_j}{\lVert x_i \rVert_2 \lVert x_j \rVert_2})\}(x_i \cdot x_j)$ for all $i, j \in [n]$. Here, $x \cdot y$ denotes the innter product between $x$ and $y$.
\end{proposition}
\begin{proof}
First, recall that $\mathcal{M} = \mathbb{E}_{g \sim N(0, I_d)}[diag[\mathbbm{1}(Xg \geq 0)] XX^{T} diag[\mathbbm{1}(Xg \geq 0)]]$. Hence, $\mathcal{M}_{ij}$ has the expression
$$
\mathcal{M}_{ij} = \mathbb{P}(x_i \cdot g \geq 0)\mathbb{P}(x_j \cdot g \geq 0)(x_i \cdot x_j).
$$
We know that for fixed $x_i, x_j$, the orthant probability is given as
$$
\mathbb{P}(x_i \cdot g \geq 0)\mathbb{P}(x_j \cdot g \geq 0) = \frac{1}{2} - \frac{1}{2\pi} \arccos(\frac{x_i \cdot x_j}{\lVert x_i \rVert_2 \lVert x_j \rVert_2}),
$$
where $g \sim \mathcal{N}(0, I_d)$. This directly implies the claim.
\end{proof}
\begin{theorem}(Theorem 3.9. of the paper)
Let each row $x_i$ of $X$ is sampled i.i.d. from $N(0, I_d)$. Furthermore, suppose $c, \delta > 0$ are given and $n/d = c$ is fixed. There exists $d_1$ such that if $d \geq d_1$, with probability at least $1 - \delta'$, we have
$$
\lambda_{min}(\mathcal{M}) \geq \frac{d}{10}.
$$
\end{theorem}
\begin{proof}
Note that $f(t) = (\frac{1}{2} - \frac{1}{2\pi}\arccos(t))t$ is not differentiable at a neighborhood around $t=1$. Hence, we need $\Tilde{f}_\epsilon$, a $C^1$ approximator of $f$, defined as

$$\Tilde{f}_\epsilon(x) = \begin{cases} f(x) \quad when \quad x \leq 1 - \epsilon\\ f'(1-\epsilon)x + f(1-\epsilon) - (1-\epsilon)f'(1-\epsilon) \quad when \quad x \geq 1 - \epsilon. \end{cases}$$
Let $\Tilde{M}$ be a matrix that satisfies
$$
[\Tilde{M}]_{ij} = \Tilde{f}_{\epsilon_0}(\frac{x_i \cdot x_j}{d}),
$$
where $\epsilon_0 = \frac{1}{100c} \leq \frac{1}{100}$. From \cref{EK}, there exists $d_2$ such that for every $d \geq d_2$, 
$$
\lambda_{min}(\Tilde{M}) \geq \lambda_{min} (\Tilde{M}') - \epsilon_0.
$$
Now, there exists $d_3$ such that when $d \geq d_3$, we can decompose $\frac{\mathcal{M}}{d}$ as
$$
\frac{\mathcal{M}}{d} = \Tilde{M} + diag(\frac{\lVert x_i \rVert_2^2}{2d} - \Tilde{f}_{\epsilon_0}(\frac{\lVert x_i \rVert_2^2}{d}))+\mathcal{M}_e,
$$
where 
$$
[\mathcal{M}_e]_{ij} = \begin{cases}
0 \quad if \quad i = j\\
\frac{1}{2\pi}(- \arccos({\frac{x_i \cdot x_j}{\lVert x_i \rVert_2 \lVert x_j \rVert_2}}) + \arccos(\frac{x_i \cdot x_j}{d}))\frac{x_i \cdot x_j}{d} \quad if \quad i \neq j. 
\end{cases}
$$
with probability at least $1 - \delta/2$. Such $d_3$ exists as for each $i \neq j$, $\frac{x_i \cdot x_j}{d} \sim \mathcal{N}(0,\frac{1}{d})$ and 
$$\mathbb{P}(\frac{|x_i \cdot x_j|}{d} \geq \frac{1}{2}) \leq 2\exp(-\frac{d}{8}),$$
meaning that for all $i \neq j$, 
$$
\mathbb{P}(\max_{i \neq j}{\frac{|x_i \cdot x_j|}{d}} \leq \frac{1}{2}) \geq 1 - 2n^2\exp(-\frac{d}{8}) = 1 - 2c^2d^2\exp(-\frac{d}{8}),
$$
and choose $d_3$ large enough so that $2c^2d^2\exp(-\frac{d}{8}) \leq \frac{\delta}{2}$. This means with probability at least $1 - \delta/2$, $f(\frac{x_i \cdot x_j}{d}) = \Tilde{f}_{\epsilon_0}(\frac{x_i \cdot x_j}{d})$ for all $i \neq j$, meaning that the off-diagonal entries of $\frac{\mathcal{M}}{d}$ can be decomposed to $\Tilde{M}$ and $\mathcal{M}_e$.
Now, Weyl's inequality leads to
\begin{align}
\label{cond1}
\frac{\lambda_{\min}(\mathcal{M})}{d} \geq \lambda_{\min}(\Tilde{\mathcal{M}}) + \min(\frac{\lVert x_i \rVert_2^2}{2d}) - \max(\Tilde{f}_{\epsilon_0}(\frac{\lVert x_i \rVert_2^2}{d})) - \lVert \mathcal{M}_e \rVert_F.
\end{align}
Then, from the concentration inequality of Chi-square random variables, we know that there exists $d_4$ such that for all $d \geq d_4$, 
\begin{equation}
\label{cond2}
(1 - \epsilon_0) d \leq (1 - \frac{\epsilon_0}{\log 2n}) d \leq \lVert x_i \rVert_2^2 \leq (1 + \frac{\epsilon_0}{\log 2n}) d \leq (1 + \epsilon_0) d, \quad \forall \quad i \in [n],
\end{equation}
with probability at least $1 - \delta/2$. Hence, when $d \geq \max\{d_2, d_3, d_4\},$ both inequalities (\ref{cond1}) and (\ref{cond2}) satisfy with probability at least $1 - \delta$. Use (\ref{cond1}) and (\ref{cond2}) to get
\begin{align*}
\frac{\lambda_{\min}(\mathcal{M})}{d} &\geq \lambda_{\min}(\Tilde{\mathcal{M}}) + \min(\frac{\lVert x_i \rVert_2^2}{2d}) - \max(\Tilde{f}_{\epsilon_0}(\frac{\lVert x_i \rVert_2^2}{d})) - \lVert \mathcal{M}_e \rVert_F\\
&\geq \lambda_{\min}(\Tilde{\mathcal{M}}) + \frac{1}{2}(1 - \epsilon_0) - \Tilde{f}_{\epsilon_0}(1+\epsilon_0) - \lVert \mathcal{M}_e \rVert_F\\
&\geq \lambda_{\min}(\Tilde{\mathcal{M}'}) - \epsilon_0 + \frac{1}{2}(1 - \epsilon_0) - \Tilde{f}_{\epsilon_0}(1+\epsilon_0) - \lVert \mathcal{M}_e \rVert_F.
\end{align*}
The last inequality follows from the fact that $d \geq d_2$. We also know that 
$$
\lambda_{\min}(\Tilde{M}') \geq (\Tilde{f}_{\epsilon_0}(1) - \Tilde{f}_{\epsilon_0}(0) - \Tilde{f}_{\epsilon_0}'(0)),
$$
as both $11^{T}, XX^{T}$ are not invertible. Moreover, $\Tilde{f}_{\epsilon_0}(0) = f(0) = 0$, $\Tilde{f}_{\epsilon_0}'(0) = f'(0) = 1/4$. Substitute to get
\begin{align*}
\frac{\lambda_{\min}(\mathcal{M})}{d} &\geq \Tilde{f}_{\epsilon_0}(1) - \frac{1}{4} - \epsilon_0 + \frac{1}{2}(1 - \epsilon_0) - \Tilde{f}_{\epsilon_0}(1+\epsilon_0) - \lVert \mathcal{M}_e \rVert_F.\\
&\geq \frac{1}{4} - \frac{3}{2}\epsilon_0 - \epsilon_0 f'(1 - \epsilon_0) - \lVert \mathcal{M}_e \rVert_F\\
&\geq 0.209 - \lVert \mathcal{M}_e \rVert_F.
\end{align*}
At last, upper bounding $\lVert \mathcal{M}_e \rVert_F$ would be enough to prove the claim. First, we know that when $|u|, |v| \leq \frac{2}{3}$, we know that 
$$|\arccos(u) - \arccos(v)| \leq |\arccos'(2/3)||u-v| \leq 2|u-v|.$$ This leads to
\begin{align*}
\lVert \mathcal{M}_e \rVert_F &\leq \frac{1}{2\pi} \sqrt{\sum_{i=1}^{n}\sum_{j \neq i} |- \arccos({\frac{x_i \cdot x_j}{\lVert x_i \rVert_2 \lVert x_j \rVert_2}}) + \arccos(\frac{x_i \cdot x_j}{d})|^2|\frac{x_i \cdot x_j}{d}|^2}\\
&\leq \frac{2}{2\pi} \sqrt{\sum_{i=1}^{n}\sum_{j \neq i} |- {\frac{x_i \cdot x_j}{\lVert x_i \rVert_2 \lVert x_j \rVert_2}} + \frac{x_i \cdot x_j}{d}|^2|\frac{x_i \cdot x_j}{d}|^2}\\
&\leq \frac{\epsilon_0}{\pi(1 - \epsilon_0)\log 2n} \sqrt{\sum_{i=1}^{n}\sum_{j \neq i}|\frac{x_i \cdot x_j}{d}|^4}
\end{align*}
with probability at least $1 - \delta$. We know from the concentration inequality of maximum of absolute values that 
\begin{equation}
\label{cond3}
\mathbb{P}(\max_{i \leq j} |x_i \cdot x_j| \leq 4\sqrt{2\log 2n}\lVert x_i \rVert_2) \geq 1 - \frac{2}{(2n)^9}, \quad \forall \quad i \in [n].
\end{equation}
Hence, with probability at least $1 - \delta - \frac{1}{(2n)^8},$ we know that both (\ref{cond2}) and
$$
\max_{i \leq j} |x_i \cdot x_j| \leq 4\sqrt{2\log 2n \cdot (1 + \epsilon_0) d}
$$
holds. Hence,
\begin{align*}
\lVert \mathcal{M}_e \rVert_F &\leq \frac{\epsilon_0}{\pi(1 - \epsilon_0)\log 2n} \sqrt{\sum_{i=1}^{n}\sum_{j \neq i}|\frac{x_i \cdot x_j}{d}|^4}\\
&\leq \frac{\epsilon_0}{\pi(1 - \epsilon_0)d^2\log 2n}\sqrt{\sum_{i=1}^{n}\sum_{j \neq i}|x_i \cdot x_j|^4}\\
&\leq \frac{\epsilon_0n}{\pi(1 - \epsilon_0)d^2\log 2n} \cdot 32\log 2n(1 + \epsilon_0)d\\
&\leq \frac{32\epsilon_0(1 + \epsilon_0)c}{\pi(1 - \epsilon_0)}\\ 
&\leq \frac{32}{100\pi}\frac{101}{99} \leq 0.104
\end{align*}
Hence, we know that with probability at least $1 - \delta'$, 
$$
\lambda_{\min}(\mathcal{M}) \geq \frac{d}{10} 
$$
holds, provided that $d \geq d_2, d_3, d_4$. Choose $d_1 \geq \max\{d_2,d_3,d_4\}$ sufficiently large.
\end{proof} 
\begin{corollary}(Corollary 3.10. of the paper)
When the conditions of \cref{mineig} holds, with probability at least $1 - \delta' - e^{-Cn}$, we have
$$
\kappa \leq 20(\sqrt{c}+1)^2,
$$
for some $C > 0$. Moreover, let $p_1^{*}$ and $\Tilde{p}_1^{*}$ be optimal values of problem \eqref{gatedconvex_simpler} with all possible hyperplane arrangements and randomly sampled arrangements, respectively.  When we sample $\Tilde{P} \geq 160(\sqrt{c}+1)^2 \log(\frac{n}{\delta})$ hyperplane arrangement patterns, we have
$$
p_1^{*} \leq \Tilde{p}_1^{*} \leq \frac{2\sqrt{10}}{G}(\sqrt{c}+1)p_1^{*}
$$
with probability at least $1 - \delta - \delta' - e^{-Cn}.$
\end{corollary}
\begin{proof}
We know that 
$$
\sqrt{\lambda_{max}(XX^T)} \leq \sqrt{2}(\sqrt{c} + 1)\sqrt{d},
$$
holds with a high probability that decays exponentially, and we may write it holds with probability at least $1 - e^{-Cn}$ for some positive constant $C$. Also, from \cref{mineig}, we know that
$$
\sqrt{\lambda_{min}(\mathcal{M})} \geq \frac{1}{\sqrt{10}}\sqrt{d}.
$$
with probability at least $1 - \delta'$. Combine the two results to obtain the wanted upper bound on $\kappa$. Also,
we know that with probability at least $1 - \delta$, 
$$
p_1^{*} \leq \Tilde{p}_1^{*} \leq \frac{\sqrt{2\kappa}}{G} p_1^{*},
$$
from \cref{t2}. Directly using the upper bound on $\kappa$ leads to the following result.
\end{proof}
\newpage
\section{Proofs in Section 4.}
\label{proof4}
\begin{proposition}(Proposition 4.2. of the paper)
For subsampled hyperplane arrangement patterns $\Tilde{D}_1, \Tilde{D}_2, ..., \Tilde{D}_{\Tilde{P}}$, let
\begin{equation}
p_0^{*} = \min_{u_i,v_i \in \mathcal{K}_{\Tilde{D}_i}} \frac{1}{2}\lVert \sum_{i=1}^{\Tilde{P}} \Tilde{D}_iX(u_i - v_i) - y \rVert_2^2 + \beta \sum_{i=1}^{\Tilde{P}} (\lVert u_i \rVert_2 + \lVert v_i \rVert_2), 
\end{equation}
\begin{equation}
\label{gatedreluconvexC}
p_1^{*} = \min_{w_i\in \mathbb{R}^{d}} \frac{1}{2}\lVert \sum_{i=1}^{\Tilde{P}} \Tilde{D}_iXw_i - y \rVert_2^2 + \beta \sum_{i=1}^{\Tilde{P}} \lVert w_i \rVert_2,
\end{equation}
and $\mathcal{K}_{\Tilde{D}_i} = \{u | (2\Tilde{D}_i - I)Xu \geq 0\}$. Suppose problem (\ref{gatedreluconvexC}) has solutions $w_i^{*}$ for $i \in [\Tilde{P}]$, and let 
$$
C = \max_{i \in [\Tilde{P}]} C(\mathcal{K}_{\Tilde{D}_i}, \frac{w_i^{*}}{\lVert w_i^{*} \rVert_2}).
$$
Then, $p_0^{*} \leq Cp_1^{*}$ holds.
\end{proposition}
\begin{proof}
For the optimum $w_i^{*}$ of problem (\ref{gatedreluconvexC}), decompose $w_i^{*}$ into $u_i^{*}, v_i^{*}$ such that $\lVert u_i^{*} \rVert_2 + \lVert v_i^{*} \rVert_2$ is minimal, $u_i^{*}, v_i^{*} \in K_i$ and $u_i^{*} - v_i^{*} = w_i^{*}$. Then,
$$
\frac{\lVert u_i^{*} \rVert_2 + \lVert v_i^{*} \rVert_2}{\lVert w_i^{*} \rVert_2} \leq C(\mathcal{K}_{\Tilde{D}_i}, \frac{w_i^{*}}{\lVert w_i^{*} \rVert_2}) \leq C
$$
and when we substitute $u_i^{*}, v_i^{*}$ in (\ref{convex1}), we can see that the result would not be greater than $Cp_1^{*}$, as the regression loss is identical and the regularization loss does not blow up $C$ times. Hence, we may conclude that $p_0^{*} \leq Cp_1^{*}$.
\end{proof}
\begin{proposition}(Proposition 4.3. in the paper)
Take any unit vector $z$ and a cone $\mathcal{K} = \{u | (2D - I) Xu \geq 0\}$. If there exists a vector $u$ that satisfies
$$\lVert u \rVert_2 \leq 1, \quad (2D - I)Xu \geq \epsilon \cdot |(2D - I)Xz|,$$
we know that
$$
C(\mathcal{K}, z) \leq 1 + \frac{1}{\epsilon}.
$$
\end{proposition}
\begin{proof}
As $(2D - I)Xu \geq \epsilon \cdot |(2D - I)Xz| \geq (2D - I)X(\pm \epsilon z),$ we know that the vectors $u + \epsilon z, u - \epsilon z \in \mathcal{K}$. Then, the two vectors 
$$
v_1 = \frac{1}{2} (\frac{u}{\epsilon} + z), \ \ v_2 = \frac{1}{2} (\frac{u}{\epsilon} - z).
$$
become two vectors in $\mathcal{K}$ that satisfy $v_1 - v_2 = z$. Hence, the cone sharpness 
$$
C(\mathcal{K}, z) \leq \lVert v_1 \rVert_2 + \lVert v_2 \rVert_2 \leq 1 + \frac{1}{\epsilon} \lVert u \rVert_2 \leq 1 + \frac{1}{\epsilon}.
$$
where the first inequality follows from the definition of $C(\mathcal{K}, z)$, the second inequality follows from triangular inequality, and the last follows from the fact that $u$ has norm no greater than 1.
\end{proof}
\begin{theorem}
(Theorem 4.4. in the paper)
Let $b \in \mathbb{R}^{n}$ sampled from a folded normal distribution, and let $X \in \mathbb{R}^{n \times {d}}$ be a matrix where each entries are sampled from a normal distribution. Consider the random variable

$$
F(X, b) = \max_{\lVert z \rVert_2 = 1} \min_{\substack{Xu \geq -Xz - kb \\ Xu \geq Xz - kb \\ k \geq 0}} \lVert u \rVert_2 + k,
$$
where $u,z \in \mathbb{R}^{d}, k \in \mathbb{R}$. Then, $\mathbb{P}_{X,b}(F(X,b) \leq 200c\sqrt{c \log 2n}) \geq 1 - 1/n^{10} - e^{-Cd}$ 
for some $C > 0$.
\end{theorem}
\begin{proof}

First, observe the inner minimization problem is strictly feasible for all $\lVert z \rVert_2 = 1$, hence strong duality holds. Now, we may write
\begin{align*}
F(X, b) &= \max_{\lVert z \rVert_2 = 1} \min_{\substack{Xu \geq -Xz - kb \\ Xu \geq Xz - kb \\ k \geq 0}} \lVert u \rVert_2 + k\\
&= \max_{\lVert z \rVert_2 = 1} \min_{k \geq 0, u} \max_{\lambda, \mu \geq 0} \lVert u \rVert_2 + k + \lambda^T(-Xu - Xz - kb) + \mu^{T}(-Xu + Xz - kb)\\
&= \max_{\substack{\lVert z \rVert_2 = 1 \\ \lambda, \mu \geq 0}} \min_{k \geq 0, u} \lVert u \rVert_2 + k + \lambda^T(-Xu - Xz - kb) + \mu^{T}(-Xu + Xz - kb)\\
&= \max_{\substack{\lVert z \rVert_2 = 1 \\ \lambda, \mu \geq 0}} \min_{k \geq 0, u} \lVert u \rVert_2 - (X^T(\lambda + \mu))^T u + k(1 - b^T(\lambda + \mu)) + (X^T(\mu - \lambda))^T z\\
&= \max_{\substack{\lambda, \mu \geq 0\\ \lVert X^T(\lambda + \mu)\rVert_2 \leq 1 \\ b^T(\lambda + \mu) \leq 1}} \lVert X^T(\lambda - \mu) \rVert_2.
\end{align*}
Let's write event $E_1$ to be:
$$
E_1 := \sigma_{\max}(X) \leq 2\sqrt{n}.
$$
Clearly $\mathbb{P}(E_1) \geq 1 - e^{-C_1n}$ for some $C_1 > 0$ by Gordon comparison \cite{thrampoulidis2014gaussian}.
Now, we write event $E_2$ to be:
$$
E_2 := \max_{\substack{\nu \geq 0\\ \lVert X^T\nu \rVert_2 \leq 1 \\ b^T\nu \leq 1}} \lVert \nu \rVert_2 \leq \frac{100\sqrt{\log 2n}}{\sqrt{d}} c = M_0.
$$
We show that $\mathbb{P}_{X, b}(E_2) \geq 1 - \frac{1}{n^{20}} - e^{-C_2d}$.
One simple fact we know is that 
$$
\min_{\substack{\nu \geq 0\\ \lVert \nu \rVert_2 = M \\ b^T\nu \leq 1}} \lVert X^{T} \nu \rVert_2 \geq 1
$$
implies
$$
\max_{\substack{\nu \geq 0\\ \lVert X^T\nu \rVert_2 \leq 1\\ b^T\nu \leq 1}} \lVert \nu \rVert \leq M.
$$
The reason is because if there exists $\nu^{*}$ that satisfies 
$$\lVert \nu \rVert_2 = M' > M, \nu^{*} \geq 0, \lVert X^{T}\nu^{*} \rVert_2 \leq 1, b^T\nu^{*} \leq 1,
$$
we can choose $\frac{M}{M'}\nu^{*}$ to find a vector that satisfies $\nu \geq 0, \lVert \nu \rVert_2 = M, b^T\nu \leq 1$ and $\lVert X^{T} \nu \rVert_2 \leq 1$, which is a contradiction.
This means when we define event $E_3$ to be
$$
E_3 := \min_{\substack{\nu \geq 0\\ \lVert \nu \rVert_2 = M_0 \\ b^T\nu \leq 1}} \lVert X^T \nu \rVert_2 \geq 1,
$$
we know $\mathbb{P}_{X, b}(E_2) \geq \mathbb{P}_{X, b}(E_3)$. 
We may write the optimization problem in $E_3$ as
$$
\min_{\substack{\nu \geq 0\\ \lVert \nu \rVert_2 = M_0 \\ b^T\nu \leq 1}} \max_{\lVert \eta \rVert_2 = 1} \eta^TX^T\nu.
$$
From Gordon comparison, when we define event $E_4$ to be
$$
E_4:= \min_{\substack{\nu \geq 0\\ \lVert \nu \rVert_2 = M_0 \\ b^T\nu \leq 1}} \max_{\lVert \eta \rVert_2 = 1} \lVert \eta \rVert_2 g^{T}\nu + \lVert \nu \rVert_2 h^{T}\eta \geq 1,
$$
where $g \sim \mathcal{N}(0, I_n), h \sim \mathcal{N}(0,I_d)$, we have that 
$$
2\mathbb{P}_{g,h,b}(E_4) \leq \mathbb{P}_{X,b}(E_3) + 1,
$$
i.e. an almost sure lower bound of the optimization problem in $E_4$ acts as an almost sure lower bound of the optimization problem in $E_3$. 
Now, define $E_5$ to be:
\begin{align*}
E_5 := \lVert h \rVert_2 \geq \sqrt{d}&(1 - \epsilon_0), \lVert g[i_1, i_2, \cdots i_d]_{-} \rVert_2 \leq \sqrt{\frac{d}{2}} ( 1 + \epsilon_0),\\
&\lVert g \rVert_\infty \leq 10\sqrt{\log (2n)}, b_{(d+1)} \geq \frac{1}{2c},
\end{align*}
where $b_{(1)} \leq b_{(2)} \leq \cdots b_{(n)}$ are the order statistics of $b$, and $b_{(m)} = b_{i_m}$ for $m = 1, 2, \cdots n$, hence $i_m$ indexes the $m$ - th order statistics. Also, $\epsilon_0$ is a fixed constant that can be taken small, and $v_{-} = v - v_{+}$ denotes the negative part of $v$. At last, $g[a_1, a_2, \cdots ,a_k]$ denotes the $k$ chosen entries of $g$, indexed with $a_i$. As $g, h, b$ are independent, we know that $E_5$ holds with high probability, i.e. $\mathbb{P}_{g,h,b}(E_5) \geq 1 - \frac{1}{n^{30}} - e^{-C_5d}$ for some $C_5 > 0$. 

We now prove that $E_5$ implies $E_4$. First write the optimization problem in $E_4$ and solve $\eta$ to get
\begin{align*}
\min_{\substack{\nu \geq 0\\ \lVert \nu \rVert_2 = M_0 \\ b^T\nu \leq 1}} \max_{\lVert \eta \rVert_2 = 1} \lVert \eta \rVert_2 g^{T}\nu + \lVert \nu \rVert_2 h^{T}\eta &= \min_{\substack{\nu \geq 0\\ \lVert \nu \rVert_2 = M_0 \\ b^T\nu \leq 1}} g^T \nu + \lVert \nu \rVert_2 \lVert h \rVert_2.\\
&\geq M_0\sqrt{d}(1 - \epsilon_0) + \min_{\substack{\nu \geq 0\\ \lVert \nu \rVert_2 = M_0 \\ b^T\nu \leq 1}} g^T \nu.
\end{align*}
The last inequality follows from $E_5$. Next, we solve over $\nu$.
\begin{align*}
\min_{\substack{\nu \geq 0\\ \lVert \nu \rVert_2 = M_0 \\ b^T\nu \leq 1}} g^T \nu 
&\geq \min_{\substack{\nu \geq 0\\ \lVert \nu \rVert_2 \leq M_0 \\ b^T\nu \leq 1}} g^T \nu\\
&= \min_{\nu} \max_{\alpha, \beta, \gamma \geq 0} g^T\nu + \alpha(b^T\nu - 1) + \beta(\lVert \nu \rVert_2 - M_0) - \gamma^{T} \nu\\
&= \max_{\alpha, \beta, \gamma \geq 0}\min_{\nu} g^T\nu + \alpha(b^T\nu - 1) + \beta(\lVert \nu \rVert_2 - M_0) - \gamma^{T} \nu\\
&\geq \max_{\substack{\alpha, \beta, \gamma \geq 0\\ \lVert g + \alpha b - \gamma \rVert_2 \leq \beta}} -(\alpha + M_0 \beta)\\
&= \max_{\alpha, \gamma \geq 0} -(\alpha + M_0\lVert g + \alpha b - \gamma \rVert_2).\\ 
\end{align*}
Now, choose $\alpha_0 = 20c\sqrt{\log (2n)}$, $\gamma_0 = (g + \alpha_0 b)_{+}$. Then, note that $g+\alpha_0 b$ has positive entries for $i \neq i_1, i_2, \cdots i_d$. That is because assuming $E_5$, $g_i \geq -10\sqrt{\log 2n}$ and $b_i \geq \frac{1}{2c}$ for $i \neq i_1, i_2, \cdots i_d$. Hence, 
$$
\lVert (g+\alpha_0 b)_{-} \rVert_2 = \lVert (g[i_1, i_2, \cdots i_d]+\alpha_0 b[i_1, i_2, \cdots i_d])_{-} \rVert_2
$$
holds. When we substitute $\alpha_0, \gamma_0$, we get
\begin{align*}
\max_{\alpha, \gamma \geq 0} -(\alpha + M_0\lVert g + \alpha b - \gamma \rVert_2)
&\geq -\alpha_0 - M_0 \lVert g + \alpha_0 b - \gamma_0 \rVert_2
\\
&= -\alpha_0 - M_0 \lVert (g + \alpha_0 b)_{-} \rVert_2\\
&=  -\alpha_0 - M_0\lVert (g[i_1, i_2, \cdots i_d]+\alpha_0 b[i_1, i_2, \cdots i_d])_{-} \rVert_2\\
&\geq -\alpha_0 - M_0 \lVert (g[i_1, i_2, \cdots i_d])_{-} \rVert_2.
\end{align*}
The last inequality follows from the fact that for positive vector $p$ and any vector $g$, $\lVert (g + p)_{-} \rVert_2 \leq \lVert g_{-} \rVert_2$ because adding the positive term only decreases the absolute value of each negative entry, and does not influence positive entries. Now, from $E_5$, we have
$$
\max_{\alpha, \gamma \geq 0} -(\alpha + M_0\lVert g + \alpha b - \gamma \rVert_2) \geq -\alpha_0 - M_0 (1 + \epsilon_0) \sqrt{\frac{d}{2}},
$$
and finally we have
\begin{align*}
\min_{\substack{\nu \geq 0\\ \lVert \nu \rVert_2 = M_0 \\ b^T\nu \leq 1}} \max_{\lVert \eta \rVert_2 = 1} \lVert \eta \rVert_2 g^{T}\nu + \lVert \nu \rVert_2 h^{T}\eta 
&\geq M_0\sqrt{d}(1 - \epsilon_0) + \min_{\substack{\nu \geq 0\\ \lVert \nu \rVert_2 = M_0 \\ b^T\nu \leq 1}} g^T \nu.\\
&\geq 100(1 - \epsilon_0) c\sqrt{\log 2n} - 20c\sqrt{\log 2n} - 100\frac{1+\epsilon_0}{\sqrt{2}} c\sqrt{\log 2n}\\
&\geq 5c\sqrt{\log 2n} \geq 1.
\end{align*}
Hence, we have shown that $E_5$ implies $E_4$. We know that 
$$
2\mathbb{P}_{g,h,b}(E_5) \leq 2\mathbb{P}_{g,h,b}(E_4) \leq 1 + \mathbb{P}_{X,b}(E_3) \leq 1 + \mathbb{P}_{X,b}(E_2),
$$
and we obtain $\mathbb{P}_{X,h}(E_2) \geq 1 - \frac{1}{n^{20}} - e^{-C_2d}$, setting $C_2 = C_5\log(2)$ and noticing $n^{30} \geq 2n^{20}$.
This means with high probability over $X, b$, 
$$
\max_{\substack{\nu \geq 0\\ \lVert X^T\nu \rVert_2 \leq 1 \\ b^T\nu \leq 1}} \lVert \nu \rVert_2 \leq \frac{100\sqrt{\log 2n}}{\sqrt{d}} c.
$$
The probability that both $E_1$ and $E_2$ will happen is at least $1 - \frac{1}{n^{20}} - e^{-Cd}$ for some $C > 0$. When both happens, $F(X, b) \leq 200c\sqrt{c}\sqrt{\log 2n}$. The reason is, for optimal $\lambda^{*}, \mu^{*}$ for the dual problem, we have
$$
\lVert \lambda^{*} + \mu^{*} \rVert_2 \leq \frac{100\sqrt{\log 2n}}{\sqrt{d}} c,
$$
and we know $\lVert \lambda^{*} - \mu^{*}\rVert_2 \leq \lVert \lambda^{*} + \mu^{*} \rVert_2$ as $\lambda^{*}, \mu^{*}$ are positive vectors. Now,
\begin{align*}
F(X, b) 
&= \lVert X^T(\lambda^{*} - \mu^{*}) \rVert_2.\\
&\leq 2\sqrt{n}\lVert \lambda^{*} + \mu^{*} \rVert_2\\
&\leq 2\sqrt{n} \frac{100\sqrt{\log 2n}}{\sqrt{d}} c = 200c\sqrt{c}\sqrt{\log 2n}.
\end{align*}
Hence $E_1, E_2$ implies $F(X, b) \leq 200c\sqrt{c}\sqrt{\log 2n}$, and we obtain the wanted result.
\end{proof}
\begin{corollary}
\label{c2}
(Corollary 4.5. of the paper) Suppose $n$, $d$ are sufficiently large that \cref{superimportantconeconstraint} holds with probability at least $1 - \delta''$, for $b \in \mathbb{R}^{n}$ sampled from a folded normal distribution and $X \in \mathbb{R}^{n \times {d-1}}$ sampled from a normal distribution. Then, with probability at least $1 - \delta''$,
$$
C(\mathcal{K}_{\Tilde{D}_i}, z) \leq 2 + 200c\sqrt{c \log 2n},
$$
also holds for all unit vectors $z$.
\end{corollary}
\begin{proof}
We prove for the cone that contains $e_1$, $\mathcal{K} = \{\bar{X}u \geq 0\}$. Take any unit vector $z$. We know that 
$$
\max_{\lVert z \rVert_2 = 1} \min_{\substack{Xu \geq -Xz - kb \\ Xu \geq Xz - kb \\ k \geq 0}} \lVert u \rVert_2 + k \leq 200c\sqrt{c \log 2n},
$$
with probability at least $1 - \delta''$. Hence, with probability $1 - \delta''$, there exists $u_0 \in \mathbb{R}^{d-1}$, $k_0 \geq 0$ that satisfies
$$
\bar{X}[1:n,2:d] u_0 + \bar{X}[1:n,1] k_0 \geq |\bar{X}[1:n,2:d]z[2:d]|.  
$$
and $u_0 + k_0 \leq 200c\sqrt{c\log 2n}$. Here, $X[a:b,c:d]$ denotes the submatrix of row $a$ to $b$, column $c$ to $d$. For that $u_0, k_0$, we know that 
$$
\bar{X}[1:n,2:d] u_0 +\bar{X}[1:n,1] (k_0+1) \geq \bar{X}[1:n,1]|z[1]| + |\bar{X}[1:n,2:d]z[2:d]| \geq |\bar{X}z|.
$$
Write $u = \begin{bmatrix}
    k_0+1\\
    u_0
\end{bmatrix}$ to see that $\bar{X}u \geq |\bar{X}z|$. Also, the norm $\lVert u \rVert_2$ is bounded by
$$
\lVert u \rVert_2 \leq 1 + k_0 + \lVert u_0 \rVert_2 \leq 1 + 200c\sqrt{c\log 2n}.
$$
Choose the center in \cref{chebyshev_sim} as $u/(1 + k_0 + \lVert u_0 \rVert_2)$ and apply the proposition to obtain the wanted result.
\end{proof}
\begin{corollary}
\label{c3}
For subsampled hyperplane arrangement patterns $\Tilde{D}_1, \Tilde{D}_2, ..., \Tilde{D}_{\Tilde{P}}$, let
\begin{equation}
\label{reluconvex}
p_0^{*} = \min_{u_i,v_i \in \mathcal{K}_{\Tilde{D}_i}} \frac{1}{2}\lVert \sum_{i=1}^{\Tilde{P}} \Tilde{D}_iX(u_i - v_i) - y \rVert_2^2 + \beta \sum_{i=1}^{\Tilde{P}} (\lVert u_i \rVert_2 + \lVert v_i \rVert_2), 
\end{equation}
\begin{equation}
p_1^{*} = \min_{w_i\in \mathbb{R}^{d}} \frac{1}{2}\lVert \sum_{i=1}^{\Tilde{P}} \Tilde{D}_iXw_i - y \rVert_2^2 + \beta \sum_{i=1}^{\Tilde{P}} \lVert w_i \rVert_2.
\end{equation}
Furthermore, assume (A1) and $n, d$ to be sufficiently large so that \cref{c2} holds with probability at least $1 - \delta''$.
Then, we have
$p_1^{*} \leq p_0^{*} \leq (2 + 200c \sqrt{c\log 2n})p_1^{*}$ with probability at least $1 - \Tilde{P}\delta''$.
\end{corollary}
\begin{proof}
The first inequality follows from triangular inequality, and the second from \cref{coneapprox}. $\Tilde{P}$ appears due to union bound, i.e. all cone sharpness constants should be bounded.
\end{proof}

\newpage
\section{Proof of the Main Theorems}
\label{Proofmain}
\begin{theorem}(Theorem 2.1. of the paper)
Let the optimal value of the 2-layer ReLU network as
$$
p^{*} = \min_{u,\alpha} \frac{1}{2}\lVert \sum_{j=1}^{m} (Xu_j)_{+} \alpha_j - y \rVert_2^2 + \frac{\beta}{2} \sum_{j=1}^{m} (\lVert u_j \rVert_2 ^2 + \lVert \alpha_j \rVert_2^2),
$$
and the convex optimization problem with random hyperplane arrangement patterns as
$$
\Tilde{p}^{*} = \min_{u_i,v_i \in \mathcal{K}_{\Tilde{D}_i}} \frac{1}{2}\lVert \sum_{i=1}^{m/2} \Tilde{D}_iX(u_i - v_i) - y \rVert_2^2 + \beta \sum_{i=1}^{m/2} (\lVert u_i \rVert_2 + \lVert v_i \rVert_2).
$$
Suppose that $m = \kappa\max\{m^{*}, 320(\sqrt{c}+1)^2 \log(\frac{n}{\delta})\}$ for fixed $\kappa \geq 1$, where $m^{*}$ defined as in \cref{notations}. Moreover, assume $n/d = c$ is fixed and the entries of $X$ are i.i.d. $\sim \mathcal{N}(0,1)$. At last, let $d \geq d_3$ so that both \cref{c3} and \cref{c3.10} holds with probability at least $1 - \delta - \delta' - m\delta''$ and $G < 1/2$. Then,
$$
p^{*} \leq \Tilde{p}^{*} \leq 2\sqrt{20}(\sqrt{c}+1)(2 + 200c\sqrt{c \log 2n })\ p^{*},
$$
holds with probability at least $1 - \delta - \delta' - m\delta''$.
\end{theorem}
\begin{proof}
Let's denote
$$
p_2^{*} = \min_{u_i,v_i \in \mathcal{K}_{D_i}} \frac{1}{2}\lVert \sum_{i=1}^{P} D_iX(u_i - v_i) - y \rVert_2^2 + \beta \sum_{i=1}^{P} (\lVert u_i \rVert_2 + \lVert v_i \rVert_2),
$$
$$
p_3^{*} = \min_{u_i \in \mathbb{R}^{d}} \frac{1}{2}\lVert \sum_{i=1}^{P} D_iXu_i - y \rVert_2^2 + \beta \sum_{i=1}^{P} \lVert u_i \rVert_2,
$$
where $D_1, D_2, ..., D_P$ are all possible hyperplane arrangement patterns, and
$$
p_4^{*} = \min_{u_i \in \mathbb{R}^{d}} \frac{1}{2}\lVert \sum_{i=1}^{m/2} \Tilde{D}_iXu_i - y \rVert_2^2 + \beta \sum_{i=1}^{m/2} \lVert u_i \rVert_2,
$$
where $\Tilde{D}_1, \Tilde{D}_2, ..., \Tilde{D}_{m/2}$ are randomly sampled hyperplane arrangement patterns. First, from preliminaries, we know that $p^{*} = p_2^{*}$. It is clear that $p_2^{*} \leq \Tilde{p}^{*}$, as we use less hyperplane arrangement patterns during approximation. Moreover, we know that $p_3^{*} \leq p_2^{*}$, as $\lVert u_i - v_i\rVert_2 \leq \lVert u_i \rVert_2 + \lVert v_i \rVert_2$ and $u_i - v_i$ can represent arbitrary vector in $\mathbb{R}^{d}$ even with the constraint $u_i, v_i \in \mathcal{K}_{\Tilde{D}_i}$. To wrap up, we know that
$$
p_3^{*} \leq p_2^{*} = p^{*} \leq \Tilde{p}^{*}.
$$
From \cref{c3} we know that $\Tilde{p}^{*} \leq (2 + 200c \sqrt{c \log 2n})\ p_4^{*}$, and we also know from \cref{prop4} that $p_3^{*} \geq G\beta \frac{\lVert y \rVert_2}{\sqrt{\lambda_{\max}(XX^{T})}}$. At last, from \cref{prop3}, we know that $p_4^{*} \leq \sqrt{2}\beta \frac{\lVert y \rVert_2}{\sqrt{\lambda_{min}(\mathcal{M})}}$. Hence, we know that
\begin{align*}
G\beta \frac{\lVert y \rVert_2}{\sqrt{\lambda_{\max}(XX^{T})}} \leq p_3^{*} \leq p_2^{*} = p^{*} \leq \Tilde{p}^{*} &\leq (2 + 200c\sqrt{c \log 2n})\ p_4^{*} \\
&\leq \sqrt{2}\beta \frac{\lVert y \rVert_2}{\sqrt{\lambda_{min}(\mathcal{M})}} \cdot (2 + 200c\sqrt{c \log 2n})\\
&\leq \frac{\sqrt{2}}{G}\sqrt{\frac{\lambda_{\max}(XX^{T})}{\lambda_{\min}(\mathcal{M})}} \cdot (2 + 200c\sqrt{c \log 2n}) \ p^{*}\\
&\leq \frac{\sqrt{20}}{G}(\sqrt{c}+1)(2 + 200c\sqrt{c \log 2n})\ p^{*},
\end{align*}
which finishes the proof that 
$$
p^{*} \leq \Tilde{p}^{*} \leq 2\sqrt{20}(\sqrt{c}+1)(2 + 200c\sqrt{c \log 2n})\ p^{*}.
$$
\end{proof}
\begin{theorem}(Theorem 2.3. of the paper)
Assume (A1) and suppose that $m = \kappa\max\{m^{*}, 320(\sqrt{c}+1)^2\log(n/\delta)\}$ for $\kappa \geq 1$. Then, there exists a randomized algorithm with $O(d^3m^3)$ complexity that solves problem $(\ref{nonconvex_original})$ within $O(\sqrt{\log n})$ relative optimality bound with high probability.
\end{theorem}
\begin{proof}
Consider the Gaussian relaxation of the convex reformulation with $[m/2]$ hyperplane arrangement patterns. As there are $O(dm)$ variables, we can solve the problem with $O(d^3m^3)$ complexity using standard interior point solvers. Moreover, we have the approximation bound of \cref{finalthm}, which leads to the fact that the solved global minima has $O(\sqrt{\log 2n})$ guarantees. At last, we can map the solution $\{(u_i^{*}, v_i^{*})\}_{i=1}^{[m/2]}$ to the parameter space of two-layer neural networks with the mapping
$$u_i^{*} \rightarrow (\frac{u_i^{*}}{\sqrt{\lVert u_i^{*} \rVert_2}}, \sqrt{\lVert u_i^{*} \rVert_2}), v_i^{*} \rightarrow (\frac{v_i^{*}}{\sqrt{\lVert v_i^{*} \rVert_2}}, -\sqrt{\lVert v_i^{*} \rVert_2}),$$
to find the parameters of the two-layer neural network that has the same loss function value as the optimal value of the convex problem. Hence, we can find parameters of two-layer neural network that has $O(\sqrt{\log n})$ relative optimality bound in $O(d^3m^3)$ time.
\end{proof}
\begin{theorem}(Theorem 2.5. of the paper)
Consider the training problem $\min_{u_j, \alpha_j} \mathcal{L}(u, \alpha),$ where the loss function $\mathcal{L}$ is given as 
$$
\mathcal{L}(u, \alpha) = \frac{1}{2} \lVert \sum_{j=1}^{m} (Xu_j)_{+}\alpha_j - y \rVert_2^2 + \frac{\beta}{2} \sum_{j=1}^{m} (\lVert u_j \rVert_2^2 + \alpha_j^2).
$$
Assume (A1),(A2), $d$ sufficiently large and $m = \kappa\max\{m^{*}, 320(\sqrt{c}+1)^2\log(\frac{n}{\delta})\}$ for some fixed $\kappa \geq 1$ so that \cref{finalthm} holds with probability at least $1 - \delta - \delta' - m\delta''$. For any random initialization $\{u_i^{0},\alpha_i^{0}\}_{i=1}^{m}$, suppose local gradient method converged to a stationary point $\{u_i',\alpha_i'\}_{i=1}^{m}$. Then, with probability at least $1 - \delta - \delta' - m\delta''$,
$$
\mathcal{L}(u',\alpha') \leq C\sqrt{\log 2n} \mathcal{L}(u^{*},\alpha^{*}),
$$
for some $C \geq 1$. Here, $\{u_i^{*},\alpha_i^{*}\}_{i=1}^{m}$ is a global optimum of $\mathcal{L}(u, \alpha)$.
\end{theorem}
\begin{proof}
We know that for the stationary point $\{u_i',\alpha_i'\}_{i=1}^{m}$, we have a corresponding convex optimization problem
\begin{equation}
\label{eqeq}
p^{*}:= \min_{u_i, v_i \in \mathcal{K}_{\Tilde{D}_i}}\frac{1}{2} \lVert \sum_{i=1}^{m} \Tilde{D}_i(u_i - v_i) - y \rVert_2^2 + \beta \sum_{i=1}^{m} \lVert u_i \rVert_2 + \lVert v_i \rVert_2,
\end{equation}
where $\Tilde{D}_i = \mathbbm{1}(Xu_i' \geq 0)$ and the solution mapping of the optimal solution $\{(u_i^{*}, v_i^{*})\}_{i=1}^{m}$ of the convex problem and the stationary point given as
$$
\alpha_i' = sign(\alpha_i')\sqrt{\lVert u_i^{*} \rVert_2}, \quad u_i' = \frac{u_i^{*}}{\sqrt{\lVert u_i^{*} \rVert_2}}\mathbbm{1}(\alpha_i'\geq 0) + \frac{v_i^{*}}{\sqrt{\lVert v_i^{*} \rVert_2}}\mathbbm{1}(\alpha_i'<0).
$$
By (A2), we know that at least half of the random hyperplane arrangement patterns are preserved from $\mathbbm{1}(Xu_i \geq 0)$ at initialization. Let the hyperplane arrangement patterns be $\Tilde{D}_1, \Tilde{D}_2, ..., \Tilde{D}_{m/2}$ without loss of generality. Now, we know that the optimal solution of the problem
$$
p' := \frac{1}{2} \lVert \sum_{i=1}^{m/2} \Tilde{D}_i(u_i - v_i) - y \rVert_2^2 + \beta \sum_{i=1}^{m/2} \lVert u_i \rVert_2 + \lVert v_i \rVert_2
$$
has $O(\sqrt{\log 2n})$ approximation guarantee with probability at least $1 - \delta - \delta' - m\delta''$, and as it is using less hyperplane arrangement patterns, we know that $p^{*} \leq p'$. At last, we know that $p^{*} = \mathcal{L}(u', \alpha')$ from solution mapping. Plugging in $\mathcal{L}(u', \alpha')$ at $p^{*}$ and using \cref{finalthm} on $p'$ yields the wanted result.
\end{proof}
\newpage
\section{Effect of Regularization}
\label{appendix:regularization}

In this section, we demonstrate that regularization may affect the test performance of the model. We do two experiments, one with a synthetic dataset with a hidden planted two-layer network with width 50, and the other with MNIST data where digits 0$\sim$4 are labeled 1 and 5$\sim$9 are labeled -1. We use the SCNN library \cite{DBLP:journals/corr/abs-2202-01331} to fit the model with an equivalent convex model with different width and regularization. We can observe that the test performance may differ $\sim$5\% for synthetic data, and $\sim$10\% for MNIST when regularization differs. Hence, a good choice of regularization matters regarding finding a good model.

\begin{figure}[H]
    \centering
     \begin{subfigure}
         \centering
         \includegraphics[width=0.48\textwidth]{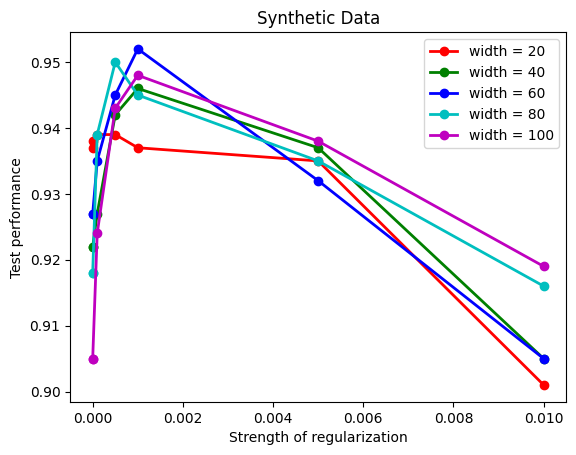}
     \end{subfigure}
     \hfill
     \begin{subfigure}
         \centering
         \hspace*{0.2cm}
         \includegraphics[width=0.48\textwidth]{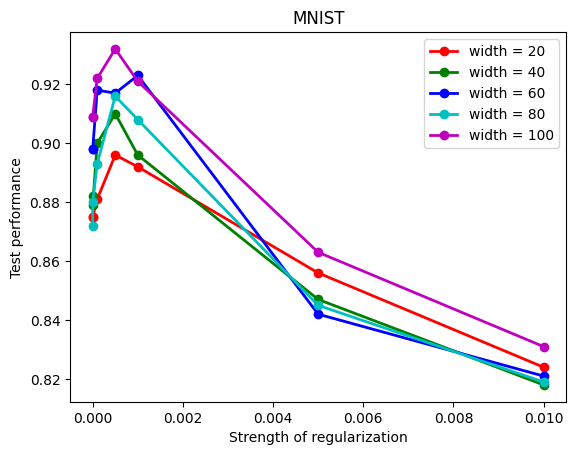}
     \end{subfigure}
        \caption{Effect of regularization for test performance. The left and the right figure shows test performance of a trained model for synthetic data and MNIST respectively, for different model sizes and regularization. For different choices of regularization, the test performance changes at maximum 5\% for synthetic data, and 10\% for MNIST.}
\end{figure}
\end{document}